\documentclass{article}
\usepackage{graphicx} 
\usepackage{fullpage}
\newcommand\truex{\boldsymbol{x^*}}
\newcommand\cov{\boldsymbol{a}}
\newcommand\zero{\boldsymbol{0}}
\newcommand\dist[2]{\mathrm{dist}\left(#1,#2\right)}
\newcommand\x{\boldsymbol{x}}
\newcommand\norm[1]{\left\lVert#1\right\rVert}
\newcommand\EE[1]{\mathbb{E}\left[#1\right]}
\newcommand\e{\boldsymbol{e}}
\newcommand\h{\boldsymbol{h}}
\newcommand\opnorm[1]{\norm{#1}_{\mathrm{op}}}
\newcommand\smallev[1]{\lambda_{\min}\left(#1\right)}
\newcommand\largeev[1]{\lambda_{\max}\left(#1\right)}
\newcommand\g{\boldsymbol{g}}
\newcommand\err{\textbf{err}}
\newcommand\covb{\boldsymbol{b}}
\newcommand\covc{\boldsymbol{c}}
\newcommand\p{\boldsymbol{p}}

\newcommand{\tr}[1]{\mathrm{tr}\left({#1}\right)}
\newcommand{\wzero}{m_0}

\newcommand\var[1]{\mathrm{Var}\left(#1\right)}
\newcommand\wtx{\widetilde{\x}}

\newcommand{\wm}{\widetilde{m}}
\usepackage{enumitem}
\usepackage{adjustbox}
\newcommand\covariance[1]{\text{Cov}\left(#1\right)}
\usepackage{amsmath}
\usepackage{amsthm}
\usepackage[english]{babel}
\newtheorem{theorem}{Theorem}[section]

\newtheorem{lemma}[theorem]{Lemma}
\newtheorem{proposition}{Proposition}[theorem]
\newtheorem{definition}{Definition}[theorem]
\newtheorem{remark}{Remark}[theorem]

\usepackage{float}
\usepackage{arydshln}
\usepackage{multirow}
\usepackage{soul}
\usepackage{tablefootnote}
\usepackage{amsfonts}
\usepackage{hyperref}

\newcommand{\jmlralgorule}{\kern2pt\hrule height.8pt depth0pt\kern2pt}

\usepackage{etoolbox}

\usepackage[round]{natbib}

\title{Robust Gradient Descent for Phase Retrieval}
\author{Alex Buna \qquad Patrick Rebeschini\\ 
Departament of Statistics, University of Oxford}
\date{}

\begin{document}

\maketitle

\begin{abstract}
  Recent progress in robust statistical learning has mainly tackled convex problems, like mean estimation or linear regression, with non-convex challenges receiving less attention. Phase retrieval exemplifies such a non-convex problem, requiring the recovery of a signal from only the magnitudes of its linear measurements, without phase (sign) information. While several non-convex methods, especially those involving the Wirtinger Flow algorithm, have been proposed for noiseless or mild noise settings, developing solutions for heavy-tailed noise and adversarial corruption remains an open challenge. In this paper, we investigate an approach that leverages robust gradient descent techniques to improve the Wirtinger Flow algorithm's ability to simultaneously cope with fourth moment bounded noise and adversarial contamination in both the inputs (covariates) and outputs (responses). We address two scenarios: known zero-mean noise and completely unknown noise. For the latter, we propose a preprocessing step that alters the problem into a new format that does not fit traditional phase retrieval approaches but can still be resolved with a tailored version of the algorithm for the zero-mean noise context.
\end{abstract}

\section{Introduction}
\label{introduction}

While the foundations of asymptotic robust statistics have been laid in the previous century by the seminal works of \citet{huber} and \cite{Hampel_1986}, advancements in modern computing naturally led to the surge of a rigorous non-asymptotic robust statistical theory. In recent years, significant breakthroughs have been made in developing methods that are statistically and computationally efficient for robust estimation and regression. These methods have been derived by both the statistics and computer science community, typically under different notions of robustness. Statisticians have focused on models with i.i.d.\@ data that exhibit heavy tails, frequently employing M-estimators or the renowned median-of-means framework \citep{catoni2011challenging,Minsker_2015}. Concurrently, computer scientists have looked into models of contamination where a portion of the data is altered by an adversary, and proposed dimension-halving and stability-based algorithms \citep{lai2016agnostic,diakonikolas2019robust}. These efforts have produced optimal estimators for the fundamental problems of mean and covariance estimation \citep{Tukey1975MathematicsAT,lugosi2017subgaussian,oliveira2022improved,abdalla2024covariance}. Later, efficient techniques for mean estimation \citep{hopkins2019mean,diakonikolas2021outlier,hopkins2021robus} and linear regression \citep{prasad2018robust,pensia2021robust} were designed, the latter aiming to recover a signal from noisy linear measurements. 

Moving away from linear models, there is limited literature on robustness results for noisy \emph{quadratic} measurements, a challenge commonly referred to as phase retrieval. The goal of this non-convex and ill-posed problem is to recover an unknown $n$-dimensional signal $\truex\in\mathbb{R}^n$ from $m$ noisy measurements $\{(\cov_j,y_j)\}_{j=1}^m$ generated as $y_j\approx (\cov_j^\top\truex)^2$, an $\varepsilon$ fraction of the sample having been  contaminated by an adversary. As this problem subsumes solving quadratic systems of equations, it has applications in engineering sciences, such as X-ray cristalography and optics, to name a few \citep{x-ray,shechtman2014phase}, and it has been studied from different perspectives within the statistical learning community, e.g. in the context of multi-armed bandits \citep{lattimore2021bandit} and deep learning \citep{hand2018phase,chen2022unsupervised}. We formally describe the model in Section \ref{preliminaries}. 

Phase retrieval has been approached by the statistical learning community using methods that can be broadly split into two categories: convex relaxations and Wirtinger Flow. We focus on the latter, as it generally enjoys better statistical and computational guarantees, while also being more amenable to methods from robust statistics. A detailed exposition of convex relaxations or methods inspired by but significantly distinct from the Wirtinger Flow, such as \citep{wang2017solving}, is beyond the scope of this paper and we redirect the reader to \citep{Fannjiang_2020}. The Wirtinger Flow is a first-order method that aims at performing gradient descent with a spectral initialisatioin, or a modified version of it, on an empirical risk associated with the phase retrieval problem. The works of \cite{chen2016solving} and \cite{zhang2017mediantruncated} describe a truncated Wirtinger Flow that is robust to outliers and bounded or deterministic noise in the responses (see the Supplementary Material for a more in-depth related work discussion and Table \ref{table:pr}). However, no phase retrieval method proposed so far is resilient to heavy-tailed noise and adversarial contamination in both the covariates and the responses. 

\subsection{Our contributions}
\label{our_contributions}
We design and analyse methodologies for the robust phase retrieval problem in two distinct situations: when the learner knows the noise has mean zero, and when the noise mean is unknown and potentially non-zero. Following a data preprocessing step (only needed in the latter situation), our methods are applicable in both scenarios and offer identical theoretical guarantees. The main framework consists of two separate procedures: first, a spectral initialisation that offers a good quality starting point for the second procedure, which is a descent-type algorithm that improves the quality of the estimate at each subsequent iteration. 

For the spectral initialisation, we offer two alternatives. The first one uses computationally efficient (i.e. runtime nearly linear in input size) stable mean estimatiors and only assumes a bounded away from zero signal-to-noise ratio, but has sample size that scales as $m_0=O(n^2\log(n))$ and it restricts the contamination $\varepsilon$ to be of order $1/n$ (Theorems \ref{thm:sym_init_mean} and B.1). Our second proposal employs covariance estimation techniques from the statistical literature \citep{oliveira2022improved,abdalla2024covariance}. Consequently, it requires a fourth-moment bound on the noise and it operates with an improved $m_0=O(n)$ samples and can accommodate a constant fraction $\varepsilon$ of contamination (Theorems \ref{thm:sym_init_cov} and B.2).

Provided with a good initialisation point, we propose an iterative algorithm with constant stepsize that can accommodate constant contamination while requesting $\wm=O(n\log(n))$ samples at each one of its $T$ iterations (Theorems \ref{thm:sym_descent} and B.3, see also Table \ref{table:pr}). Although the sample complexity (total number of sampled used) of the iterative procedure is $O(Tn\log(n))$, a tuning of $T=\widetilde O(1)$ (with hidden log factors independent of the dimension $n$) reduces the sample complexity to $O(n\log(n))$, c.f.\ Remark \ref{remark:T}. 

\begin{table*}[t]\centering
\caption{Summary of the main results regarding the use of Wirtinger Flow (WF) for solving the phase retrieval (PR) problem. As detailed in Section \ref{preliminaries}, an adversary contaminates an $\varepsilon$ fraction of samples from the noisy PR model $y=(\cov^\top\truex)^2+z$, with noise variance $\sigma^2$. The goal of the learner is to find $\x$ that minimises the estimation error $\norm{\x\pm\truex}/\norm{\truex}$, decomposable into two errors of different types: the `Statistical error' caused by the noise $z$ and the contamination level $\varepsilon$; and the `Optimisation error' of the iterative procedure used, in this case gradient descent (GD)/WF. All works use spectral initialisation (c.f.\ Procedure {\protect\hyperlink{proc:spectral_init_sym}{1}}) followed by a GD in $T$ steps (c.f.\ Algorithm {\protect\hyperlink{alg:descent_sym}{2}}
) and we report in this table the statistical error only, as the optimisation error is $O(\exp(-T))$ in all works (including ours). In the third row, the dotted line (whenever present) reflects differences between the efficient mean-estimation based spectral initialisation we use (left of dotted line) and the robust GD we consider (right of dotted line).}
\renewcommand{\arraystretch}{0.5}
\AtBeginEnvironment{tabular}{\tiny}
\centering
\resizebox{\linewidth}{!}{%
\begin{tabular}{|c:c|c:c|c:c|c|} 
 \hline
    \multicolumn{2}{|c}{Method} & \multicolumn{2}{|c|}{\begin{tabular}{@{}c@{}} Contamination $\varepsilon$\\\& noise $z$\end{tabular}} & \multicolumn{2}{c|}{\begin{tabular}{@{}c@{}}
         Sample  \\
         complexity 
    \end{tabular}} & Statistical error\\ 
 \hline\hline
\multicolumn{2}{|c}{{\begin{tabular}{@{}c@{}} $\ell_2$ loss, Vanilla GD\\with sprectral init.\\ \citep{Candes_2015}\end{tabular}}} & \multicolumn{2}{|c|}{$\varepsilon=0$ \& $z=0$} & \multicolumn{2}{c|}{$n$} & $0$ (no noise) 
\\
\hline
\multicolumn{2}{|c}{\begin{tabular}{@{}c@{}} reshaped $\ell_2$ loss \emph{\underline{or}} \\ Poisson log-likelihood \\ truncated WF\\with spectral init. \\ \citep{zhang2017mediantruncated} \end{tabular}} & \multicolumn{2}{|c|}{\begin{tabular}{@{}c@{}} $\varepsilon\propto 1$ in $y$\\ \& bounded noise \\$|z|\leq B\norm{\truex}^2$ 
\end{tabular}} & \multicolumn{2}{c|}{$n\log(n)$} & \begin{tabular}{@{}c@{}}
    $\sqrt{B}(O(1)+\sqrt{\varepsilon})$ \\ \emph{\underline{or}} \\$B(O(1)+\sqrt{\varepsilon})$ 
\end{tabular}
 \\
 \hline
   \begin{tabular}{@{}c@{}}
      Spectral init \\ based on \\ mean estim.
 \end{tabular}& \begin{tabular}{@{}c@{}}
          $\ell_2$ loss  \\
          robust GD
 \end{tabular} & $\varepsilon\propto n^{-1}$ & $\varepsilon\propto 1$ & \multirow{2}{*}{$n^2\log(n)$} & \multirow{2}{*}{$n\log(n)$}  & \multirow{2}{*}{$\frac{\sigma}{\norm{\truex}^2}(O(1)+\sqrt{\varepsilon})$} \\
\multicolumn{2}{|c}{\textbf{Thms. \ref{thm:sym_init_mean} \& \ref{thm:sym_descent}}} & \multicolumn{2}{|c|}{\begin{tabular}{@{}c@{}}
     \textbf{in $\cov$ and} $y$  \\
      \textbf{\& heavy-tailed} $z$,\\ $\norm{\truex}^2/\sigma>0$
\end{tabular}} &  & &  \\ 
 \hline
 \multicolumn{2}{|c}{{\begin{tabular}{@{}c@{}} $\ell_2$ loss, robust GD\\with sprectral init.\\based on cov. estim.\\ \textbf{Thms. \ref{thm:sym_init_cov} \& \ref{thm:sym_descent}}\end{tabular}}} & \multicolumn{2}{|c|}{\begin{tabular}{@{}c@{}}
      $\varepsilon\propto 1$ \textbf{in $\cov$ and $y$}  \\
       \textbf{\& heavy-tailed $z$},\\
      $\mathbb{E}[z^4]<\infty$
 \end{tabular}
 } & \multicolumn{2}{c|}{$n\log(n)$} & $\frac{\sigma}{\norm{\truex}^2}(O(1)+\sqrt{\varepsilon})$ 
 \\
 \hline
 
\end{tabular}%
}%
\label{table:pr}
\end{table*}
Focusing on the case when the noise mean is known to be zero, the high-level idea of our iterative algorithm is as follows: We start by considering the population risk under the squared loss $r(\x)=\mathbb{E}[((\cov^\top\truex)^2-y)^2]/4$, associated to (\ref{eq:pr}). In a ball $\mathcal{B}_{\pm}$ around the unknown minimisers $\pm\truex$, the population risk is strongly-convex and smooth (Lemma \ref{lemma:localscs}). By standard results in convex optimisation, vanilla gradient descent on the population risk initialised inside $\mathcal{B}_{\pm}$ and run with an appropriately-chosen step size, is guaranteed to converge at a linear rate to one of $\pm\truex$. However, gradients of the population risk are not available and we estimate them using robust gradient estimators (Definition \ref{def:robg}). Three main challenges arise:
\begin{itemize}[leftmargin=*]
    \item The need for a robust spectral initialisation: both the center $\pm\truex$ and the radius of $\mathcal{B}_{\pm}$ depend on $\pm\truex$ and we need to guarantee that the first iterate falls within $\mathcal{B}_{\pm}$. Previous works \citep{candes2012solving,Ma_2019} achieve this by noting that $\pm\truex$ is the principal eigenvector of $\mathbb{E}[y\cov\cov^\top]$. While this idea can be exploited in conjunction with mean estimators to give rise to a polynomial time algorithm, (see Procedure {\protect\hyperlink{proc:spectral_init_sym}{1}} with configuration \texttt{MeanEstStab}  and Theorem \ref{thm:sym_init_mean}), the resulting sample size and contamination level behave worse than what one would expect from the phase retrieval and robust statistics literatures. Instead, we identify a covariance matrix (of $y\cov$) whose leading eigenvector is also $\pm\truex$, and we propose using covariance estimators that, although not efficiently computable, gives better statistical guarantees (c.f.\ Procedure {\protect\hyperlink{proc:spectral_init_sym}{1}} with configuration \texttt{CovEst} and Theorem \ref{thm:sym_init_cov}).
    \item Step-size tuning: it depends on the strong-convexity and smoothness parameters of the population risk in $\mathcal{B}_{\pm}$. However, these are again quantities that depend on the unknown $\norm{\truex}$. Instead, we use the norm of the initial iterate, which is guaranteed to be close to $\norm{\truex}$, to tune the learning rate.
    \item Ensuring iterates stay within $\mathcal{B}_{\pm}$: we use a generic \emph{stable} mean estimator for the gradient estimation and the strict contraction property of vanilla gradient descent on strongly-convex and smooth functions. See Lemma \ref{lemma:rgd_contraction} and Proposition \ref{prop:dkp}.
\end{itemize}
None of the above-mentioned challenges were real concerns in previous works that use robust gradient descent, such as \citep{prasad2018robust,liu2019high}, as they were applying it to functions that are \emph{globally} strongly-convex and smooth and these parameters did not depend on the norm of the true signal.

To overcome the problem of not having information on the noise mean, we employ an extra data preprocessing step: we split the sample in two and `subtract' from the points in the first half points from the second half. The new data, although coming from a model $\upsilon\approx \covb^\top\truex(\truex)^\top\covc$ that falls outside the scope of phase retrieval, has a zero mean noise. We note that in the rest of this paper, we present this preprocessing step as doubling the sample to avoid technicalities with odd-sized samples, but this does not change our results. More details can be found in Subsection \ref{sub:non-zero_noise}.

\subsection{Notation and organisation of the paper}

For any positive integer $k$, we use $[k]$ to refer to the set $\{1,2,\ldots,k\}$. We use lowercase bold-faced letters to denote vectors and uppercase letters to denote matrices. The vector with all entries equal to zero is denoted by $\zero$. The vector $\e_i$ has $1$ in the $i$'th coordinate and $0$ everywhere else. For a vector $\x$, we denote by $\x_i$ its $i$'th entry and by $\norm{\x}$ its Euclidean norm. We reserve the notation $A\succeq 0$ to mean that the square matrix $A$ is positive semi-definite, $A\succeq B$ if $A-B\succeq 0$ and $A \preceq B$ if $B\succeq A$. The smallest and largest eigenvalue of a symmetric matrix $A$ are denoted by $\smallev{A}$ and $\largeev{A}$, respectively. The operator norm of a matrix $A$ is denoted by $\opnorm{A}$, its Frobenius norm by $\norm{A}_{\text{F}}$. For a function $f:\mathbb{R}^d\rightarrow\mathbb{R}$, we write $\nabla f$ and $\nabla^2 f$ for its gradient and Hessian, respectively. $C,C_1,C_2,\ldots$ are absolute numerical constants. Finally, we use the big-O asymptotic notations: $f=O(g)$ if $f(x)/g(x)$ is upper-bounded by a constant for sufficiently large or small $x$ (depending on the context), $f=\Omega(g)$ if $g=O(f)$, $f=\Theta(g)$ if both $f=O(g)$ and $g=O(g)$ hold, and $\widetilde{O}(\cdot)$ hides poly-logarithmic factors. For a scalar random variable $X$ we will denote by $\var{X}$ its variance. In case this is multidimensional, we will denote its covariance by $\covariance{X}$.

Section \ref{preliminaries} introduces formally the model of robust phase retrieval and presents in more technical detail the main ideas employed for known mean zero noise. The concrete procedures and the main theorems for the known zero-mean noise can be found in Section \ref{main_results}. Finally, we present our conclusion in Section \ref{conclusion}. All proofs, along with the unknown mean case main results are deferred to the Supplementary Material.

\section{Setting and main ideas}

\label{preliminaries}

\subsection{Robust phase retrieval}
\label{sub:model}
The setting of our problem and the made assumptions are as follows: For a fixed and unknown $n$-dimensional vector $\truex\in\mathbb{R}^n$, consider the statistical model
\begin{align}
    y=(\cov^\top\truex)^2+z, \label{eq:pr}
\end{align}
where the covariate $\cov\in\mathbb{R}^n$ is distributed as a standard normal $\mathcal{N}(\zero,I_n)$ and the noise has an \emph{unknown} mean, \emph{unknown} variance $\sigma^2$, unknown bounded central fourth moment $K_4^4=\mathbb{E}[(z-\mathbb{E}[z])^4]<\infty$, and is independent of $\cov$. A learning algorithm can request $m$ samples. Then, data $S=\{(\cov_j,y_j)\}_{j=1}^m$ is generated according to (\ref{eq:pr}) and passed onto an adversary that inspects the sample $S$ and returns to the learner a dataset $T\subset \mathbb{R}^n\times\mathbb{R}$ such that $|T|=m$ and $|T\cap S|\geq (1-\varepsilon)m$. The contamination level $\varepsilon\in(0,1/2)$ is known to the learner.

The goal of a learning algorithm is to find, given access to the previously described process, some $\x\in\mathbb{R}^n$ that minimises the following notion of distance:
\begin{align}
    \dist{\x}{\truex}:=\min\left\{\norm{\x-\truex},\norm{\x+\truex}\right\}. \label{dist}
\end{align}
The distance is defined as in (\ref{dist}) because it is impossible to distinguish between $\truex$ and $-\truex$. 

Assuming the covariates $\cov$ are initially sampled from $\mathcal{N}(\zero,I_n)$ is standard in the phase retrieval literature, see, for example, \citep{Candes_2015}. Additionally, the assumption that the fourth moment of the noise is finite is, in fact, needed in only one of our results (Theorem \ref{thm:sym_init_cov}) and has been made in many previous robust statistics works (for example, \citep{pensia2021robust,oliveira2022improved}). The signal-to-noise ratio $\norm{\truex}^2/\sigma$ and its fourth-moment counterpart $\norm{\truex}^2/K_4$ are assumed to be a positive constant and in particular they do not depend on the dimension $n$. The model implicitly assumes that the algorithm can collect data at any moment. This is common across learning theory, for example in the definition of a PAC learning example oracle \citep{valiant}. 

\subsection{Zero mean noise}

For the rest of this subsection, we assume that $\mathbb{E}[z]=0$ and the learner knows this.

\subsubsection{Population risk local geometry}

Consider the $\ell_2$ population risk associated to (\ref{eq:pr}): 
\begin{align}
    r(\x)=\EE{((\cov^\top \x)^2-y)^2}/4, \label{population_risk}
\end{align}
where the expectation is taken over $(\cov,y)$. Given that $\mathbb{E}[z]=0$, it is the case that $\pm\truex$ are amongst the minimisers of the population risk. As mentioned in subsection \ref{our_contributions}, our algorithm aims at simulating gradient descent on the population risk (\ref{population_risk}) to recover one of these minimisers. In turn, this will be achievable as a consequence of the local geometry. We recall some basic definitions from the convex analysis literature. We primarily work with twice differentiable functions. 

\begin{definition}
    Let $\mathcal{W}\subseteq \mathbb{R}^n$ be a convex set. A twice differentiable function $f:\mathcal{W}\rightarrow\mathbb{R}$ is said to be \emph{$\alpha$-strongly convex} for some $\alpha>0$ if $\nabla^2 f \succeq \alpha I_n$, and \emph{$\beta$-smooth} for some $\beta>0$ if  $\nabla^2 f\preceq \beta I_n$.
\end{definition}

\begin{lemma}
\label{lemma:localscs}
    The population risk (\ref{population_risk}) is $\alpha:=4\norm{\truex}^2$-strongly convex and $\beta:={73}\norm{\truex}^2/9$-smooth in a ball around $\truex$ of radius $R:=\norm{\truex}/9$. That is:
    \begin{align*}
        4\norm{\truex}^2 I_n\preceq \nabla^2 f(\x)\preceq ({73}\norm{\truex}^2/9) I_n\quad,
    \end{align*}
    for all $\x\in\mathbb{R}^n$ with $\norm{\x-\truex}\leq \norm{\truex}/9$.
\end{lemma}

In the rest of this subsection, we will use the specific values of $\alpha$, $\beta$, and $R$ as prescribed in Lemma \ref{lemma:localscs}. Of course, an analogous result holds for $-\truex$ by symmetry. Given the favorable geometry of the population risk, consider running gradient descent initialised at some $\x_0$ that satisfies $\norm{\x_0-\truex}\leq R$, with step-size $\eta=2/({\alpha+\beta})$, whose iterate at step $t+1$ is
\begin{align}
    \x_{t+1}=\x_t-\eta\nabla r\left(\x_t\right).\label{eq:gd}
\end{align}
This will converge at a linear rate to $\truex$, a well-known result in convex optimisation, see for example \citep{bubeck2015convex}. The only subtlety is in making sure that each iterate stays in the region with this geometry, and this holds because the distance from iterates to $\truex$ will strictly contract at each step.

\subsubsection{Robust gradient descent}

Leaving the question of initialisation aside, one generally does not have access to the population risk or its gradients. The insight from \citep{chen2017distributed,holland2018efficient,prasad2018robust} is that exchanging integration and differentiation, (\ref{eq:gd}) becomes
\begin{align*}
    \x_{t+1}=\x_t-\eta\EE{((\cov^\top \x_t)^2-y)(\cov^\top \x)\cov},
\end{align*}
and now the unknown object is the expectation of the random variable $((\cov^\top \x_t)^2-y)(\cov^\top \x)\cov$. In turn, this can be replaced by accurate robust and heavy-tailed mean estimators:

\begin{definition}
    \label{def:robg}
    For a sample $T=\{(\cov_i,y_i)\}_{i=1}^m$ of size $m$, we say that $\g(\,\cdot\,;T,\delta,\varepsilon)$ is a \emph{gradient estimator} if there exists functions $A,B\!:\!\mathbb{N}\!\times\![0,1]^2\!\rightarrow\!\mathbb{R}$ such that, for any fixed $\x\in\mathbb{R}^n$, with probability at least $1-\delta$,
    \begin{align}
        \norm{\g\!\left(\x;T,\delta,\varepsilon\right)\!-\!\nabla r(\x)}\!\!\leq \!\!A(m,\delta,\varepsilon)\!\norm{\x\!-\!\truex\!}\!+\!B(m,\delta,\varepsilon). \label{ineq:grad_est}
    \end{align}
\end{definition}

Instead of using $\nabla r(\x_t)$ as the descent direction, our algorithms will follow instead $\g_t:=g(\x_t;T,\delta,\varepsilon)$, where $T$ is a sample of size $m$, $\delta$ is a fixed confidence parameter and $\varepsilon$ is the known contamination level in the sample $T$. We will require that the sample $T$ is fresh at each iteration to be able to control the deviation bound (\ref{ineq:grad_est}) for any fixed (and thus independent of $T$) point $\x$. Access to fresh samples at each iteration has appeared before in the robust gradient descent literature, see for example \citep{prasad2018robust,merad2023robust}. The new iterate is obtained as
\begin{align}
    \x_{t+1}=\x_t-\eta \g_t. \label{eq:rgd}
\end{align}
The guarantees of new iterate $\x_{t+1}$ are described in the following Lemma:
\begin{lemma}
\label{lemma:rgd_contraction}
    Suppose $\x_t$ obeys $\norm{\x_t-\truex}\leq R$ and $\eta\leq2/({\alpha+\beta})$. Then, with probability at least $1-\delta$, the new iterate $\x_{t+1}$ obtained according to (\ref{eq:rgd}) satisfies
    \begin{align}
        \norm{\x_{t+1}-\truex}\leq \biggl(\sqrt{1-\frac{2\eta\alpha\beta}{\alpha+\beta}}+\eta A(m,\delta,\varepsilon)\biggr)\norm{\x_t-\truex}+\eta B(m,\delta,\varepsilon).\label{gd:semi-contraction}
    \end{align}
\end{lemma}


Lemma \ref{lemma:rgd_contraction} can be then applied inductively as long as the right-hand side of (\ref{gd:semi-contraction}) is at most $R$. In turn, this will hold if $A(m,\delta,\varepsilon)$ and $B(m,\delta,\varepsilon)$ are small enough, for which an appropriate gradient estimator has to be chosen. We recall a result of \cite{diakonikolas2021outlier}:
\begin{proposition}
\label{prop:dkp}
    [Proposition 1.5 in \citep{diakonikolas2021outlier}] Let $T$ be an $\varepsilon$-corrupted set of $m$ samples from a distribution in $\mathbb{R}^n$ with mean $\boldsymbol{\mu}$ and covariance $\Sigma$. Let $r(\Sigma)=\tr{\Sigma}/{\opnorm{\Sigma}}$ and $\varepsilon'=\Theta(\log(1/\delta)/m+\varepsilon)\leq c$ be given, for a constant $c>0$. Then, any \emph{stability-based algorithm} on input $T$ and $\varepsilon'$, efficiently computes $\boldsymbol{\widehat{\mu}}$ such that, with probability at least $1-\delta$,
    \begin{align*}
        \norm{\boldsymbol{\widehat{\mu}}-\boldsymbol{\mu}}=O\Bigl(&\sqrt{{\tr{\Sigma}\log r(\Sigma)}/{m}}+\sqrt{\opnorm{\Sigma}\varepsilon}+\sqrt{{\opnorm{\Sigma}\log(1/\delta)}/{m}}\Bigr).
    \end{align*}
\end{proposition}

By now, there is an entire plethora of stability-based algorithms that can be used for mean estimation, the main difference between them being their runtime. We will simply use a generic stability algorithm \texttt{MeanEstStab}, and discuss a specific choice when talking about our algorithms' runtimes (c.f.\ Remark \ref{remark:mean_computational_complexity}).


\subsubsection{Spectral initialisation}

We need to ensure that the initial iterate $\x_0$ satisfies $\norm{\x_0-\truex}\leq R$ (or, symmetrically, $\norm{\x_0+\truex}\leq R$). The idea used in \citep{Candes_2015} and \citep{Ma_2019} is the following: the leading eigenpair of $\mathbb{E}[{y\cov\cov^\top}]=\norm{\truex}^2I_n+2\truex(\truex)^\top$ is $(\pm\truex,3\norm{\truex}^2)$.
So, to recover the direction and scale of $\truex$, it suffices to compute the leading eigenpair for an estimate of the mean of the random matrix $y\cov\cov^\top$. However, we need to use a robust mean estimator instead of the matrix empirical mean, and such estimators were primarily designed for vectors. In order to overcome this issue, we propose two alternatives: 
\begin{enumerate}
    \item One can estimate the vectors $\mathbb{E}[y(\cov^\top\e_i)\cov]=\mathbb{E}[y\cov\cov^\top]\e_i$ for each $i\in [n]$ using \texttt{MeanEstStab} and stack these as the columns of the estimate for $\mathbb{E}[y\cov\cov^\top]$. While this will give a polynomial time algorithm with the right choise of \texttt{MeanEstStab}, aggregating the estimated columns would result in a suboptimal sample size and contamination level, scaling as $n^2\log(n)$ and $1/n$, respectively.
    \item Instead, one can note that $\covariance{y\cov}=(3\norm{\truex}^4+\sigma^2)I_n+12\norm{\truex}^2\truex(\truex)^\top$, whose leading eigenvector is again $\pm\truex$. This allows for the use of covariance estimators for corrupt and heavy-tailed (fourth moment bounded) data, which will improve the sample complexity to $O(n)$ and allow for a constant (i.e. not dependent on the dimension $n$) contamination level. We generically call such estimators \texttt{CovEst} and present their statistical guarantees in the following:
\end{enumerate}
\begin{proposition}
    \label{prop:cov_est}
    [Theorem 1.3 in \citep{oliveira2022improved}, see also Theorem 1 in \citep{abdalla2024covariance}] Let $T$ be an $\varepsilon$-corrupted set of $m$ samples from a distribution in $\mathbb{R}^n$ of a random variable $\boldsymbol{X}\in\mathbb{R}^d$ with mean $0$ and covariance matrix $\Sigma\in\mathbb{R}^{n\times n}$, satisfying $\mathbb{E}[\norm{\boldsymbol{X}}^4]<\infty$. Denote by
    \begin{align*}
        \kappa_4=\underset{\boldsymbol{v}\in\mathbb{R}^n,\boldsymbol{v}^\top\Sigma\boldsymbol{v}=1}{\sup}\mathbb{E}[(\boldsymbol{v}^\top\boldsymbol{X})^4]^{1/4}.
    \end{align*}
    Fix the confidence $1-\delta\in(0,1)$. Let $r(\Sigma)=\tr{\Sigma}/\opnorm{\Sigma}$ and suppose $\varepsilon=O(\kappa_4^{-4})$, $ m=\Omega(r(\Sigma)+\log(1/\delta))$. There is an estimator \texttt{CovEst} depending on $\delta$, $m$ and $\varepsilon$ that takes as input the corrupted sample $T$ and outputs $\widehat{\Sigma}\in\mathbb{R}^{n\times n}$ such that, with probability at least $1-\delta$,
    \begin{align*}
        \opnorm{\widehat{\Sigma}-\Sigma}=O\Bigl(\kappa_4^2\opnorm{\Sigma}(\sqrt{r(\Sigma)/m}+\sqrt{\varepsilon}+\sqrt{\log(1/\delta)/m})\Bigl).
    \end{align*}
\end{proposition}
Recovering the scale of $\truex$ can be done by noting that $\mathbb{E}[y]=\norm{\truex}^2$, and thus a robust estimate for $\norm{\truex}$ can be obtained using \texttt{MeanEstStab} on the responses $y$.

Finally, recall that the step-size we use is $\eta\!\leq\! 2/(\!\alpha\!+\!\beta\!)=\Theta(\norm{\truex}^{-2})$. As a consequence of a spectral initialisation we also have $\norm{\truex}\approx\norm{\x_0}$, so we can use $\norm{\x_0}$ to find an appropriate value for $\eta$.

\subsection{Unknown mean noise}
\label{sub:non-zero_noise}
For the case of noise with an unknown (and potentially non-zero) mean, the idea described in the previous subsection does not work anymore, as the posterior mean of the population risk (\ref{population_risk}) will depend on the noise variance, which is unknown. Our proposed solution is reducing this problem to one that falls outside the scope of phase retrieval, but in which the noise has again mean zero. 

More specifically, suppose we have two samples $y=(\cov^\top\truex)^2+z$ and $y'=((\cov')^\top\truex)^2+z'$. Subtracting one from the other and rescaling leads to
\begin{align*}
    \underbrace{\frac{y-y'}{2}}_{:=\upsilon}={\underbrace{\Bigl(\frac{\cov+\cov'}{\sqrt{2}}\Bigl)}_{:=\covb}}^{\!\!\top}\truex(\truex)^\top\underbrace{\Bigl(\frac{\cov-\cov'}{\sqrt{2}}\Bigl)}_{:=\covc}+\underbrace{\frac{z-z'}{2}}_{:=\zeta}.
\end{align*}
Generally, from an (uncorrupted) sample $S=\{(\cov_j,y_j)\}_{j=1}^{2m}$ of size $2m$ following (\ref{eq:pr}) with unknown and potentially non-zero mean noise, we obtain a sample $S'=\{(\covb_j,\covc_j,\upsilon_j)\}_{j=1}^m$ from the model
\begin{align}
    \upsilon=\covb^\top\truex(\truex)^\top\covc+\zeta, \label{eq:new_model}
\end{align}
where now $\covb,\covc\sim\mathcal{N}(\zero,I_d)$ are independent, and the noise $\zeta$ has (known) mean zero and variance $\sigma^2/2$ and is independent of $\covb$ and $\covc$. The new sample is obtained as $\covb_j=(\cov_j+\cov_{m+j})/\sqrt{2}$, $\covc_j=(\cov_j-\cov_{m+j})/\sqrt{2}$, and $\upsilon_j=(y_j-y_{m+j})/2$. Note that, when the initial sample of size $2m$ is $\varepsilon$-corrupted, the new sample of size $m$ will be $2\varepsilon$-corrupted.

A similar preprocessing step has been applied by \cite{pensia2021robust} for linear regression. However, unlike in our case, the resulting data in their work is still following a linear regression. The new model (\ref{eq:new_model}) falls outside the scope of phase retrieval and can be seen as a restricted case of blind deconvolution \citep{ahmed2018blind}. While the population risk associated with blind deconvolution does not exhibit local strong convexity and smoothness \citep{Ma_2019}, in our case (\ref{eq:new_model}) does:

\begin{lemma}
    The population risk associated to (\ref{eq:new_model}),
    \begin{align}
        r_{\mathrm{new}}(\x)=\EE{(\x^\top\covb\covc^\top\x-\upsilon)^2}/2, \label{eq:pop_risk_non-zero}
    \end{align}
    is $\norm{\truex}^2$-strongly convex and $49\norm{\truex}^2/12$-smooth in a ball around $\truex$ of radius $\norm{\truex}/6$.
\end{lemma}

We defer this case to the Supplementary Material, where we show that the same type of procedures and analysis as in the case of known zero mean noise can be applied to this scenario after the preprocessing step, and the statistical and computational guarantees are identical up to universal constants.

\section{Main results}

\label{main_results}

In this section, we assume the learner knows that the noise in the underlying robust phase retrieval model has mean zero. We begin with Procedure \hyperlink{proc:spectral_init_sym}{1}, which outputs, using a spectral method, the initial iterate $\x_0\in\mathbb{R}^d$ of the robust gradient descent. It takes as inputs a confidence level $\delta\in(0,1)$, the contamination parameter $\varepsilon$ and it is allowed access to samples from the robust phase retrieval model. This is reminiscent of Algorithm 1 in \citep{Ma_2019}, but now two configurations are possible: either the \texttt{MeanEstStab} algorithm or \texttt{CovEst} estimator is being used to estimate the mean of a random matrix or covariance of a random vector. As attested by Theorems \ref{thm:sym_init_mean} and \ref{thm:sym_init_cov}, the property of Procedure \hyperlink{proc:spectral_init_sym}{1}
's output (in either configuration) is that with high probability it lies in a region of strong convexity and smoothness of Lemma \ref{lemma:localscs}.
\begin{figure}[H]
\centering
\begin{minipage}{\columnwidth}
\jmlralgorule\par\smallskip
\textbf{Procedure \hypertarget{proc:spectral_init_sym}{{1}}:} Spectral initialisation for robust phase retrieval with zero mean noise\smallskip
\jmlralgorule\par\smallskip
\textbf{Inputs:} $\delta\in(0,1)$, $\varepsilon>0$, access to robust phase retrieval data

\textbf{Output:} $\x_0\in\mathbb{R}^n$

1. Receive a sample $T=\{(\cov_j,y_j)\}_{j=1}^{m_0}$ of size $m_0$ from the robust phase retrieval model.

\dotfill

\underline{Either} \emph{ (Algorithmic) \emph{\texttt{MeanEstStab}} configuration:}

2. For $j\!\in\![\wzero]$ and $i\!\in\! [n]$, let $\p_{ij}=y_j(\cov_j^\top\e_i)\cov_j$. Let $\boldsymbol{\mu}^{(i)}$ be the \texttt{MeanEstStab} estimate for $\{\p_{ij}\}_{j=1}^{m_0}$, $\boldsymbol{{Y}}\in\mathbb{R}^{n\times n}$ the matrix whose $i$'th column is $\boldsymbol{\mu}^{(i)}$.

\dotfill

\underline{Or} \emph{\emph{\texttt{CovEst}} configuration:}

2. Let $\boldsymbol{Y}\!\!\in\!\mathbb{R}^{n\times n}$ be the \texttt{CovEst} estimate for $\{\!y_j\cov_j\!\}_{j=1}^{m_0}$.

\dotfill

3. Let $\widetilde{y}^2$ be the \texttt{MeanEstStab} output on $\{y_j\}_{j=1}^{m_0}$. If $\widetilde{y}^2<0$, return $\boldsymbol{0}$. Else, set $\x_0$, normalised to $\norm{\x_0}=\sqrt{\widetilde{y}^2}$, to be an eigenvector corresponding to the largest eigenvalue of $(\boldsymbol{Y}+\boldsymbol{Y}^\top)/2$. Return $\x_0$.

\smallskip\jmlralgorule
\end{minipage}
\end{figure}




\begin{theorem}
    \label{thm:sym_init_mean}
    (\emph{\texttt{MeanEstStab}} Spectral initialisation) Let $\truex\!\in\!\mathbb{R}^n$ be arbitrary and $\delta\!\in\!(0,1)$. There exist universal constants $C_1,C_2$ such that, when the contamination level satisfies $\varepsilon\!\leq\! C_1(\norm{\truex}^4\!/\sigma^2)n^{-1}$, the sample size is $m_0\geq C_2\max\{n^2\log(n),n\log(n/\delta)\}\sigma^2/\norm{\truex}^4$, and Procedure \hyperlink{proc:spectral_init_sym}{1} with configuration \emph{\texttt{MeanEstStab}}  is run on inputs $\delta$, $\varepsilon$ and requests a sample of size $m_0$ from the robust phase retrieval with mean zero noise model, it returns $\x_0\!\in\!\mathbb{R}^n$ such that, with probability at least $1-\delta$,
    \begin{align*}
        \dist{\x_0}{\truex}\leq\norm{\truex}/9.
    \end{align*}
\end{theorem}

\begin{remark}
    \label{remark:mean_sample_contamination}
    \emph{(Sample size and contamination leve)} As prescribed by Theorem \ref{thm:sym_init_mean}, Procedure \hyperlink{proc:spectral_init_sym}{1} with configuration \emph{\texttt{MeanEstStab}} requires $m_0=O(n^2\log(n))$ samples and tolerates an amount of corruption $\varepsilon=O(1/n)$ in order to guarantee a `good-quality' initial iterate $\boldsymbol{x}_0$. In particular, one would expect a better dependence on the ambient dimension $n$ for both of these based on previous results in the respective literatures (i.e. $m_0=O(n)$ as in \citep{Candes_2015,Ma_2019} and $\varepsilon=O(1)$ (in $n$)). While these quantities can be improved in our setting by using a covariance estimator (c.f.\ Theorem \ref{thm:sym_init_cov}), there are two advantage of using the \texttt{MeanEstStab}. Firstly, this procedure configuration works even for noise that does not have a bounded fourth moment, but only a bounded away from zero signal-to-noise ratio. This primarily happens because, unlike in Proposition \ref{prop:cov_est}, there are no finite fourth moment assumptions in Proposition \ref{prop:dkp}. The second advantage comes from a computational viewpoint, as outlined in the following Remark \ref{remark:mean_computational_complexity}.
\end{remark}

\begin{remark}
    \label{remark:mean_computational_complexity}
    \emph{(Computational complexity)} The heavier computation takes place in Step 2 of Procedure \hyperlink{proc:spectral_init_sym}{1} with \emph{\texttt{MeanEstStab}} configuration: $n$ uses of \emph{\texttt{MeanEstStab}} are being made, each time for $m_0$ vectors in $\mathbb{R}^{n}$. So Step 2's computational cost is $n$ times larger than the runtime of \emph{\texttt{MeanEstStab}} on an input of size $nm_0$ and thus it is entirely determined by the choice of a stable algorithm. For example, the one designed by \cite{hopkins2021robus} runs in nearly linear time $\widetilde{O}(m_0n)$. It is based on a matrix multiplicative update algorithm and, for brevity, we omit a detailed description.\footnote{As noted after Lemma 3.2 in \citep{hopkins2021robus}, the input $\rho$ of the algorithm can be set as the squared diameter of the data.} According to Proposition \ref{prop:dkp}, we need to ensure that $\varepsilon'=\Theta(\log(1/\delta)/m+\varepsilon)\leq c$ for a constant $c>0$. This is the case when used to implement \emph{\texttt{MeanEstStab}} in Procedure \hyperlink{proc:spectral_init_sym}{1} (and, below, in Algorithm \hypertarget{alg:descent_sym}{{2}}) because the signal-to-noise ratio $\norm{\truex}^2/\sigma$ is assumed to be constant. Furthermore, the reason why we use $(\boldsymbol{Y}+\boldsymbol{Y}^\top)/2$ in Step 3, instead of $\boldsymbol{Y}$, is that it is easier to compute eigenpairs for the former symmetric matrix than singular pairs for the latter. This step can be executed via the power method in time $\widetilde{O}(1)$.
\end{remark}

\begin{theorem}
    \label{thm:sym_init_cov}
    (\emph{\texttt{CovEst}} Spectral initialisation) With the \emph{\texttt{CovEst}} configuration of Procedure \hyperlink{proc:spectral_init_sym}{1}, but an improved sample size of $m_0\geq C_2 \max\{n,\log(1/\delta)\}K_4^4/\norm{\truex}^8$ and a contamination level independent of the ambient dimension $\varepsilon\leq C_1(\norm{\truex}^8/K_4^4)$, the same conclusion as in Theorem \ref{thm:sym_init_mean} holds.
\end{theorem}

\begin{remark}
    In comparison to the guarantees of Theorem \ref{thm:sym_init_mean}, the spectral initialisation sample size is now of order $O(n)$, comparable to previous (but much less general, see the discussion in Subsection \ref{our_contributions} and the Supplementary Material) works on phase retrieval, and the contamination is constant in $n$, as it is the case in the robust statistics literature. Also, while in Theorem \ref{thm:sym_init_mean} the sample size and contamination level were functions of the signal-to-noise ratio $\norm{\truex}^2/\sigma$, now they depend on $\norm{\truex}^2/K_4$, a fourth-moment analogue. Specifically, this new ratio is a result of computing the value $\kappa_4$ from Proposition \ref{prop:cov_est} in our specific phase retrieval instance, whereas in Theorem \ref{thm:sym_init_mean} the signal-to-noise ratio dependence was a consequence of the factor $\opnorm{\Sigma}^{1/2}$ in Proposition \ref{prop:dkp}. Moreover, the assumption that the noise has a bounded fourth moment made in Subsection \ref{sub:model} is due to the data distribution requirement in Proposition \ref{prop:cov_est}.
\end{remark}

\begin{remark}
\label{remark_cov}
     Although, to the best of our knowledge, there are currently no efficient methods for covariance estimation, in fact one does not necessarily need to deploy covariance estimators for the spectral initialisation task. An efficient and statistically optimal algorithm that computes the leading eigenvector of a covariance matrix (i.e. combining Steps 2 and 3) would suffice. This latter task is a significant and independent topic of interest in robust statistics, with applications in other areas not directly related to phase retrieval, such as robust PCA. Enhancements in these tools would directly find applications within the setting we consider. At the level of generality we consider, fully addressing the simpler setting of linear regression has demanded considerable effort from the robust statistics community (c.f.\ Section \ref{introduction}).
\end{remark}

Next, this initial iterate $\x_0$ is then passed to the iterative Algorithm \hyperlink{alg:descent_sym}{2}. It takes as inputs a confidence parameter $\delta\in(0,1)$, the contamination level $\varepsilon$, the number of iterations $T>0$, and it is allowed access to robust phase retrieval data. Its theoretical guarantees are presented in Theorem \ref{thm:sym_descent}.

\begin{figure}[H]
\centering
\begin{minipage}{\columnwidth}
\jmlralgorule\par\smallskip
\textbf{Algorithm \hypertarget{alg:descent_sym}{{2}}:} Gradient descent for robust phase retrieval with zero-mean noise\smallskip
\jmlralgorule\par\smallskip
\textbf{Inputs:} $\x_0\in\mathbb{R}^n$, $\delta\in(0,1)$, $\varepsilon>0$, $T\in\mathbb{N}$.

\textbf{Output:} $\x_T\in\mathbb{R}^n$

1. Set $\eta=128/(981\norm{\x_0}^2)$. For $t=0,\ldots,T-1$:
    \begin{itemize}[leftmargin=*]
        \item Receive a sample $B_t=\{(\cov_j,y_j)\}_{j=1}^{\wm}$ of size $\wm$ from the robust phase retrieval model.
        \item Gradient estimation: For each $(\cov_j,y_j)\in B_t$, let $\p_j=((\cov_j^\top\x_j)^2-y_j)(\cov_j^\top \x_j)\cov_j$ and $\g_t$ to be the \texttt{MeanEstStab} output for these $\wm$ points.
        \item Update $\x_{t+1}=\x_t-\eta \g_t$.
    \end{itemize}
2. Return $\x_T$.
    \smallskip\jmlralgorule
\end{minipage}
\end{figure}

\begin{theorem}
    \label{thm:sym_descent}
    (Iterative scheme) Let $\truex\in\mathbb{R}^n$ be an arbitrary vector and $\delta\in(0,1)$. There exist universal constants $C_1,C_2,C_3,C_4$ such that, when $\x_0\in\mathbb{R}^n$ is in a ball around $\pm\truex$ of radius $\norm{\truex}/9$, the contamination level is $\varepsilon\leq C_1\norm{\truex}^4/\sigma^2$, the sample size is $T\wm\geq C_2T\max\{n\log(n),\log(1/\delta)\}\sigma^2/\norm{\truex}^4$, and Algorithm \hyperlink{alg:descent_sym}{2} is run on inputs $\x_0$, $\delta$, $\varepsilon$, $T$ and requests a sample of size $\wm$ from the robust phase retrieval with mean zero noise model at each iteration, it returns $\x_T\in\mathbb{R}^n$ such that, with probability at least $1-T\delta$,
    \begin{align*}
        \frac{\dist{\x_T}{\truex}}{\norm{\truex}}&\leq C_3\exp\bigl(-C_4\eta T\norm{\truex}^2(1-\sqrt{\varepsilon})\bigl)+C_3\frac{\sigma}{\norm{\truex}^2}\biggl(\sqrt{\frac{n\log(n)}{\wm}}+\sqrt{\frac{\log(1/\delta)}{\wm}}\biggl)+C_3\frac{\sigma}{\norm{\truex}^2}\sqrt{\varepsilon}.
    \end{align*}
\end{theorem}

\begin{remark}
As in \citep{chen2016solving,zhang2017mediantruncated}, the step-size in Algorithm \hyperlink{alg:descent_sym}{2} does not depend on the ambient dimension, in contrast to \citep{Candes_2015,Ma_2019}. The estimation error is upper-bounded by a sum of an optimisation term and two statistical terms, one controlled by the sample size and the other by the contamination level:
    \begin{itemize}[leftmargin=*]
        \item As $\eta=\Theta(\norm{\truex}^{-2})$, the optimisation error decays as $\exp(-O(T))$. As expected, when the corruption level $\varepsilon$ increases, the convergence speed decreases.
        \item The first statistical error is the product between $\sigma/\norm{\truex}^2$, the inverse of the signal-to-noise ratio, and a slow rate term $\wm^{-1/2}$. If there is no corruption ($\varepsilon=0$), the smaller $\sigma/\norm{\truex}^2$ is, the easier to recover $\truex$ one expects to be. Indeed, if we consider $\sigma\rightarrow 0$, the statistical error becomes $0$. This is also the case in the infinite sample regime $\wm/(n\log(n))\rightarrow\infty$.
        \item Finally, the error induced by the adversary is the product between the inverse of the signal-to-noise ratio and $\sqrt{\varepsilon}$. The larger this product is, the more noise and contamination exist in the data, making the problem harder. We expect this term to be optimal based on the following information-theoretic lower-bound: estimating the mean of a distribution with variance $\sigma^2$ from an $\varepsilon$-corrupted sample incurs an error of at least $\Omega(\sigma\sqrt{\varepsilon})$. 
    \end{itemize}
\end{remark}

\begin{remark}
    \label{remark:T}
    (Dependence on $T$) The dependence on the number of iterations $T$ in both sample complexity ($T\wm$) and confidence level ($1-T\delta$) is a common characteristic in works employing robust gradient descent techniques, e.g.\ \citep{prasad2018robust,liu2019high,gaïffas2022robust,merad2023robust}. This dependence arises from the use of fresh samples at each iteration, which in turn guarantees that the descent directions used by Algorithm \hyperlink{alg:descent_sym}{2} are gradient estimators (c.f.\ Definition \ref{def:robg}). Note that the optimisation and statistical errors scale as $\exp(-O(T))$ and $O(\sigma/\norm{\truex}^2)$, respectively, when $\wm =\Omega(n\log(n))$ samples are used per iteration. Setting $T=\log(C'\sigma/\norm{\truex}^2)$, for some suitable constant $C'$, the these two errors become of the same order. Under the assumption that the signal-to-noise ratio $\norm{\truex}^2/\sigma$ does not depend on the ambient dimension $n$, we have that $T=\widetilde{O}(1)$ leads to an overall sample complexity of order $n\log(n)$ for Algorithm \hyperlink{alg:descent_sym}{2}. As a function of $n$, our iterative procedure has the same sample size as the ones proposed in \citep{chen2016solving,zhang2017mediantruncated}, while being robust to a more general contamination model, c.f.\ Table \ref{table:pr}.
    
\end{remark}

\begin{remark}
\label{remark:cc}
    (Computational complexity) Recall that the sample size is $T\wm$. One can use for \texttt{MeanStabEst} the algorithm of \cite{hopkins2021robus}, which runs in nearly linear time (up to logarirthmic factors, see also Remark \ref{remark:mean_computational_complexity}). This is then used $T$ times on $\wm$ points in Step 2 of Algorithm \hyperlink{alg:descent_sym}{2}. So, Algorithm \hyperlink{alg:descent_sym}{2} essentially has the computational complexity required to read the data.
\end{remark}

\section{Conclusion}

\label{conclusion}

To the best of our knowledge, our work represents the first results on robustness for phase retrieval with heavy-tailed noise and contamination in both the covariates and responses. This setting is \emph{general} and imposes minimal assumptions on the data-generating mechanism. Fully addressing a simpler scenario in the context of linear regression has required considerable effort from the community (c.f. Section \ref{introduction}). We consider two settings: the learner either knows that the noise has mean zero, or has no information about the noise mean, which can potentially be non-zero. To tackle the latter case we propose in the Supplementary Material a data preprocessing step that reduces phase retrieval to a different statistical model that can be seen as a restriction of blind deconvolution. The population risk landscape associated with both models enjoys favorable geometry around the global minima. We use this property to run spectrally initialised robust gradient descent. The main appeal of this method is that it is as versatile as the estimator used: by choosing mean estimators resilient to both heavy tails and adversarial corruption, we obtain descent directions that are simultaneously robust to both types of outliers. The results we establish within our general robustness setting involve a contamination level that scales as $1/n$ for the computationally tractable initialisation step based on mean estimation, where $n$ is the ambient dimension of the problem, and as a constant for the spectral initialisation based on covariance estimation, and as well a constant for the iterative procedure. Further improvements (either statistical or from an algorithmic perspective) in the initialisation scheme necessitate significant advances in leading eigenvector estimators for covariance matrices, which is a topic of independent interest with wider applications in other areas, such as robust PCA (c.f. Remark \ref{remark_cov}).\looseness=-1

\newpage
\bibliography{yourbibfile}

\newpage

\begin{center}
  \LARGE Supplementary Material
\end{center}
\bigskip
\section{Related work}

\label{related work}

\textbf{Robust gradient descent} aims at performing gradient descent on the population risk associated with a learning problem. However, instead of using the inaccessible population risk gradients, it replaces them with robust mean estimates. Amongst the first uses of this method can be traced back to the works of \cite{chen2017distributed} and \cite{holland2018efficient}. The former uses smaller batches of data to construct gradients that are further aggregated using the geometric median-of-means, while the latter replaces the usual empirical risk gradient with a weighted version. \citep{yin2020defending} applies the robust gradient descent framework with different mean estimators (median, trimmed mean, iterative filtering) to non-convex problems, but in the Byzantine learning framework, which differs from ours, as we allow each fresh sample to contain outliers, rather than some samples to be clean and some to be completely corrupted. \citep{prasad2018robust} and \citep{pmlr-v97-diakonikolas19a} employ different estimators for the robust gradient that can deal with large classes of convex problems, while \citep{liu2019high} analyses the high-dimensional scenario under sparsity constraints. This latter work allows the contamination level to depend on the sparsity of the underlying solution, the number of samples and the dimension of the problem. By using robust coordinate descent, \citep{gaïffas2022robust} designs algorithms for robust linear learning with almost the same run-time and sample complexity as the non-robust counterparts. Finally, \citep{merad2023robust} improves on the results of \cite{liu2019high} by using robust mirror descent, while also studying other linear learning problems.

\textbf{Wirtinger Flow for phase retrieval} performs gradient descent with spectral initialisation on empirical objectives with favorable geometry (i.e. strongly convex and smooth regions) around the true signal $\pm\truex$. It was first proposed in \citep{Candes_2015} to study \emph{exact} phase retrieval, i.e. when there is no noise or contamination in the responses. At a very high level, it considers the unregularised empirical risk under the $\ell_2$ loss, which is used to direct the vanilla descent procedure with step size proportional to $1/n$ towards one of the global minima. The sample size and iteration complexity (i.e. number of steps) needed to guarantee an accuracy of $\tau$ are $n\log n$ and $n\log(1/\tau)$, respectively. The step size and iteration complexity for the same vanilla gradient descent procedure have been further improved in \citep{Ma_2019} to $1/\log(n)$ and $\log(n)\log(1/\tau)$, respectively, by noticing the implicit bias induced by the Wirtinger Flow to a so-called region of incoherence and contraction. Leaving the exact phase retrieval model aside, various truncation procedures have been proposed to make the Wirtinger Flow robust to bounded and possibly contamination in the responses \citep{chen2016solving,zhang2017mediantruncated}. While these works improve the sample complexity to $n$, the iteration complexity to $\log(1/\tau)$, and the step size to a constant, the trimming procedures applied to the gradient are rather ad-hoc. This makes the analyses of their theoretical guarantees hard to translate to more general types of noise, such as heavy-tailed, and more powerful adversaries that can alter the covariates as well. A detailed summary can be found in Table 1 in the Main Paper. Statistical guarantees for the $\ell_2$ empirical risk minimiser of phase retrieval with heavy-tailed noise have previously been obtained in \citep{chen2022error}, but these do not assume any contamination in the data or offer a computational framework. Another approach to solving the noiseless phase retrieval problem when only the responses are contaminated can be found in \citep{duchi2018solving} and it is based on composite optimisation. Rather than employing (a form of) gradient-descent, the authors resort to a prox-linear algorithm and convex programming. Finally, we mention that there are various extensions of Wirtinger Flow for sparse phase retrieval that replace the gradient descent procedure with mirror descent and an appropriately chosen mirror map, see, for example, the work of \cite{wu2021nearly} and the reference therein.

\section{Unknown mean noise main results}

\label{app:non-zero}

In this section, the assumption on the noise to have mean zero is dropped, yet the variance $\sigma^2$ is still bounded and unknown. The learner does not know the mean of the noise anymore but rather works with a modified dataset generated from the new statistical model (8) with zero mean noise.

We first present the spectral initialisation Procedure \hyperlink{proc:descent_sym}{3}. As before, this has two configurations: either the \texttt{MeanEstStab} algorithm or \texttt{CovEst} estimator is being used. In the first case, it recovers the direction of $\truex$ using an column-wise estimate of $\mathbb{E}[\upsilon\covb\covc^\top]=\truex\truex^\top$, whose the top eigenvalue is aligned with $\truex$. In the second case, it estimates directly $\covariance{\nu b}=(\norm{\truex}^4+\sigma^2/2)I_n+2\norm{\truex}^2\truex(\truex)^\top$, again whose leading eigenvector is a multiple of $\truex$. Then, it does an appropriate rescaling by using \texttt{MeanEstStab} for $\mathbb{E}[\nu\covb^\top\covc]=\norm{\truex}^2$. Its theoretical guarantees are given in Theorems \ref{thm:non-zero_init_mean} and \ref{thm:non-zero_init_cov}.

\begin{figure}[H]
    \centering
\begin{minipage}{\textwidth}
\jmlralgorule\par\smallskip
\textbf{Procedure \hypertarget{proc:spectral_init_non-zero}{3}:} Spectral initialisation for robust phase retrieval with unknown mean noise
\jmlralgorule\par\smallskip
\textbf{Inputs:} $\delta\in(0,1)$, $\varepsilon>0$, access to robust phase retrieval data

\textbf{Output:} $\x_0\in\mathbb{R}^n$

1. Receive a sample $T=\{(\cov_j,y_j)\}_{j=1}^{2m_0}$ of size $2m_0$ from the robust phase retrieval model. Construct a new dataset: for each $j\in[m_0]$, $$\covb_j=(\cov_j+\cov_{m_0+j})/\sqrt{2}, \covc_j=(\cov_j-\cov_{m_0+j})/\sqrt{2},\quad \text{and}\quad \upsilon_j=(y_j-y_{m_0+j})/2.$$

\dotfill

\underline{Either} (\emph{Algorithmic}) \texttt{MeanEstStab} \emph{configuration}:

2. For $j\in[\wzero]$ and $i\in [n]$, let $\p_{ij}=\upsilon_j(\covc_j^\top\e_i)\covb_j$. Let $\boldsymbol{\mu}_{i}$ be the \texttt{MeanEstStab} estimate for $\{\p_{ij}\}_{j=1}^{m_0}$ and let $\boldsymbol{Y}\in\mathbb{R}^{n\times n}$ be the matrix whose $i$'th column is $\boldsymbol{\mu}_{i}$.

\dotfill

\underline{Or} \texttt{CovEst} \emph{configuration}:

2. Let $\boldsymbol{Y}\in\mathbb{R}^{n\times n}$ be the \texttt{CovEst} estimate for $\{\nu_j\covb_j\}_{j=1}^{m_0}$.

\dotfill

3. Let $\widetilde{y}^2$ be the \texttt{MeanEastStab} output on $\{\nu \covb^\top\covc\}_{j=1}^{m_0}$. If $\widetilde{y}^2<0$, return $\boldsymbol{0}$. Else, set $\boldsymbol{x}_0$, normalised to $\norm{\boldsymbol{x}_0}=\sqrt{\widetilde{y}^2}$, to be an eigenvector corresponding to the largest eigenvalue of $(\boldsymbol{Y}+\boldsymbol{Y}^\top)/2$. Return $\boldsymbol{x}_0$.
Return $\x_0$.
    \smallskip\jmlralgorule
\end{minipage}
\end{figure}
\begin{theorem}
    \label{thm:non-zero_init_mean}
    (\texttt{MeanEstStab} Spectral initialisation) Let $\truex\in\mathbb{R}^n$ be an arbitrary vector and $\delta\in(0,1)$. There exist universal constants $C_1,C_2$ such that, when the contamination level satisfies $\varepsilon/2\leq C_1(\norm{\truex}^4/\sigma^2)n^{-1}$, the sample size is $2m_0\geq C_2\max\{n^2\log(n),n\log(n/\delta)\}\sigma^2/\norm{\truex}^4$, and Procedure \hypertarget{proc:spectral_init_non-zero}{3} with configuration \texttt{MeanEstStab} is run on inputs $\delta$, $\varepsilon$ and requests a sample of size $2m_0$ from the robust phase retrieval model (with unknown noise mean), it returns $\x_0\in\mathbb{R}^n$ such that, with probability at least $1-\delta$,
    \begin{align*}
        \dist{\x_0}{\truex}\leq\norm{\truex}/6.
    \end{align*}
\end{theorem}

\begin{theorem}
    \label{thm:non-zero_init_cov}
    (\emph{\texttt{CovEst}} Spectral initialisation) With the \emph{\texttt{CovEst}} configuration of Procedure \hypertarget{proc:spectral_init_non-zero}{3}, but an improved sample size of $m_0\geq C_2 \max\{n,\log(1/\delta)\}K_4^4/\norm{\truex}^8$ and a contamination level independent of the ambient dimension $\varepsilon\leq C_1(\norm{\truex}^8/K_4^4)$, the same conclusion as in Theorem \ref{thm:non-zero_init_mean} holds.
\end{theorem}

Note that Theorems \ref{thm:non-zero_init_mean} and \ref{thm:non-zero_init_cov} have essentially the same guarantees as Theorems 3.1 and 3.2, and the extra symmetrisation in Step 2 does increase the run-time.

These proofs of these two theorems are identical to the ones of Theorems 3.1 and 3.2, which can be found in Section \ref{appendix_main_results}, and we omit them. The only difference is that now, instead of employing the expressions computed in Lemma \ref{lemma:moments_symmetric}, the ones from Lemma \ref{lemma:moments_non-symmetric} should be used.

Next, we present Algorithm \hyperlink{alg:descent_sym}{4} that starts the iterative procedure from the output of Procedure \hypertarget{proc:spectral_init_non-zero}{3}. Its guarantees are given in Theorem \ref{thm:non-zero_descent}.

\begin{figure}[H]
\centering
\begin{minipage}{\textwidth}
\jmlralgorule\par\smallskip
\textbf{Algorithm \hypertarget{alg:descent_nonzero}{4}:} Gradient descent for robust phase retrieval with unknown mean noise
\jmlralgorule\par\smallskip
\textbf{Inputs:} $\x_0\in\mathbb{R}^n$, $\delta\in(0,1)$, $\varepsilon>0$, $T\in\mathbb{N}$.

\textbf{Output:} $\x_T\in\mathbb{R}^n$

1. Set $\eta=1024/(1647\norm{\x_0}^2)$. For $t=0,\ldots,T-1$:
    \begin{itemize}[leftmargin=*]
        \item Receive a sample $B_t=\{(\cov_j,y_j)\}_{j=1}^{2\wm}$ of size $2\wm$ from the robust phase retrieval model.
        \item Construct a new dataset: for each $j\in[\wm]$, $$\covb_j=(\cov_j+\cov_{\wm+j})/\sqrt{2}, \covc_j=(\cov_j-\cov_{\wm+j})/\sqrt{2} \quad \text{and}\quad \upsilon_j=(y_j-y_{\wm+j})/2.$$
        \item Gradient estimation: For each $(\cov_j,y_j)\in B_t$, let $\p_j=(\x_t^\top\covb\covc^\top\x_t-\upsilon_j)(\covb\covc^\top+\covc\covb^\top)\x_t$. Let $\g_t$ be the \texttt{MeanEstStab} estimate for these $\wm$ points.
        \item Update $\x_{t+1}=\x_t-\eta \g_t$.
    \end{itemize}
2. Return $\x_T$.
\smallskip\jmlralgorule
\end{minipage}
\end{figure}

\begin{theorem}
    \label{thm:non-zero_descent}
    (Iterative scheme) Let $\truex\in\mathbb{R}^n$ be an arbitrary vector and $\delta\in(0,1)$. There exist universal constants $C_1,C_2,C_3,C_4$ such that, when $\x_0\in\mathbb{R}^n$ is in a ball around $\pm\truex$ of radius $\norm{\truex}/6$, the sample size satisfies $2T\wm\geq C_2T\max\{n\log(n),\log(1/\delta)\}\sigma^2/\norm{\truex}^4$, the contamination level satisfies $\varepsilon/2\leq C_1\norm{\truex}^4/\sigma^2$, and Algorithm \hyperlink{alg:descent_sym}{4} is run on inputs $\x_0$, $\delta$, $\varepsilon$, $T$ and requests a sample of size $2\wm$ from the robust phase retrieval model (with unknown mean noise) at each iteration, it returns $\x_T\in\mathbb{R}^n$ such that, with probability at least $1-T\delta$,
    \begin{align*}
        \frac{\dist{\x_T}{\truex}}{\norm{\truex}}&\leq C_3\exp\left(-C_4\eta T\norm{\truex}^2(1-\sqrt{\varepsilon})\right)\\&+C_3\frac{\sigma}{\norm{\truex}^2}\left(\sqrt{\frac{n\log(n)}{\wm}}+\sqrt{\frac{\log(1/\delta)}{\wm}}\right)
        +C_3\frac{\sigma}{\norm{\truex}^2}\sqrt{\varepsilon}.
    \end{align*}
\end{theorem}

The proof of Theorem \ref{thm:non-zero_descent} is very similar to the proof of Theorem 3.3, which can be found in Section \ref{appendix_main_results}, and hence we omit it.

Once again, the algorithmic framework for robust phase retrieval with unknown mean noise is to first run Algorithm \hyperlink{proc:spectral_init_non-zero}{3} and pass its output to Algorithm \hyperlink{alg:descent_sym}{4}. Finally, Theorem \ref{thm:non-zero_descent} has the same guarantees (up to absolute constants) as Theorem 3.3. In conclusion, the robust phase retrieval problem does not become harder if the learner does not know that the mean of the underlying noise is zero. 

\section{Missing proofs from Section \ref{preliminaries}}

\subsection{Proof of Lemma 2.1}
\begin{proof}
    Although similar bounds (modulo constants) have appeared in literature before (see, for example, \citep{Ma_2019} Subsection 2.2), we give a proof for completeness: 
    Let $\h=\x-\truex$, with $\norm{\h}\leq R\norm{\truex}$. According to Lemma \ref{lemma:grad_and_hessian}, the Hessian becomes:
    \begin{align*}
        \nabla^2r(\x)=6\h\h^\top+6\h(\truex)^\top+6\truex\h^\top+4\truex(\truex)^\top+3\norm{\h}^2I_n+6(\h^\top\truex)I_n+2\norm{\truex}^2I_n.
    \end{align*}
    The largest eigenvalue of the Hessian is bounded by its operator norm:
    \begin{align*}
        \largeev{\nabla^2r(\x)}&\leq\opnorm{\nabla^2r(\x)}\leq 9\norm{\h}^2+18\norm{\truex}\norm{\h}+6\norm{\truex}^2\\
        &\leq (9R^2+18R+6)\norm{\truex}^2,
    \end{align*}
    where for the first inequality we have used the triangle and Cauchy-Schwarz inequalities and the fact that for rank one matrices, $\opnorm{\boldsymbol{u}\boldsymbol{v}^\top}=\norm{\boldsymbol{u}}{\norm{\boldsymbol{v}}}$.

    Next, we lower-bound $\smallev{\nabla^2 r(\x)}$ in the region $\norm{\h}\leq R\norm{\truex}$. For this, we proceed as follows:
    \begin{align*}
        (\truex)^\top\nabla^2r(\x)\truex&=6(\h^\top\truex)^2+18(\h^\top\truex)\norm{\truex}^2+3\norm{\truex}^2\norm{\h}^2+6\norm{\truex}^4\\
        &\geq \norm{\truex}^2\left(6\norm{\truex}^2-18\norm{\h}\norm{\truex}\right)\\
        &\geq \norm{\truex}^4(6-18R).
    \end{align*}
    where for the first inequality we have dropped two positive terms and use Cauchy-Schwarz. Thus,
    \begin{align*}
        \smallev{\nabla^2 r(\x)}\geq(6-18R)\norm{\truex}^2.
    \end{align*}
    The conclusion follows for $R=\frac{1}{9}$.
\end{proof}

\subsection{Proof of Lemma 2.2}
\begin{proof}
    This is similar to the proof of Theorem 1 in \citep{prasad2018robust}, which we reproduce here for convenience.
    Let $\err_t=\g_t-\nabla r(\x_t)$, so that $\norm{\err_t}\leq A(m,\delta)\norm{\x_t-\truex}+B(m,\delta)$ with probability at least $1-\delta$. On this event,
    \begin{align}
        \norm{\x_{t+1}-\truex}&=\norm{\x_t-\eta\g_t-\truex}=\norm{\x_t-\truex-\eta\left(\nabla r(\x_t)-\nabla r(\truex)\right)-\eta\err_t} \label{eq:gd_first}\\
        &\leq \norm{\x_t-\truex-\eta\left(\nabla r(\x_t)-\nabla r(\truex)\right)}+\eta\norm{\err_t} \label{eq:gd_second}\\
        &\leq \sqrt{1-\frac{2\eta\alpha\beta}{\alpha+\beta}}\norm{\x_t-\truex}+\eta\left(A(m,\delta,\varepsilon)\norm{\x_t-\truex}+B(m,\delta,\varepsilon)\right)\label{eq:gd_third}\\
        &=\left(\sqrt{1-\frac{2\eta\alpha\beta}{\alpha+\beta}}+\eta A(m,\delta)\right)\norm{\x_t-\truex}+\eta B(m,\delta). \nonumber
    \end{align}
    In (\ref{eq:gd_first}) we have used the update (4), $\nabla r(\truex)=0$, and the definition of $\err_t$. In (\ref{eq:gd_second}), we have used the triangle inequality. In (\ref{eq:gd_third}) we have used the fact that $\norm{\x_t-\truex}\leq R$, in which region the population risk is $\alpha$-strongly convex and $\beta$-smooth, hence the descent step is contractive (see, for instance, Theorem 3.12 in \citep{bubeck2015convex}), and the definition of the gradient estimator. 
\end{proof}


\subsection{Proof of lemma 2.3}
\begin{proof}
    The proof is identical to the one of Lemma 2.1, except the expressions for the Hessian of the risk, which is now given by Lemma \ref{lemma:grad_and_hessian_non-zero}.
\end{proof}

\section{Missing proofs from Section \ref{main_results}}

\label{appendix_main_results}

\subsection{Proof of Theorem 3.1.}
\label{proof:thm3.1.}
\begin{proof}
     Recall that the noise is has mean 0. We split the proof into three parts: recovering the direction (Step 2 of Algorithm 1 with configuration \texttt{StabMeanEst}) and the scale (Step 3) of $\truex$, and finally obtaining guarantees for the output of Algorithm 1.

    \bigskip
    \noindent
    \textbf{Recovering the direction of $\truex$:} Let $\Sigma_i:=\covariance{y(\cov^\top \e_i)\cov}$. By the guarantees of Proposition 2.1.1, for each $i\in[n]$ with probability at least $1-\delta/(2n)$,
    \begin{align}
        \norm{\boldsymbol{\mu}_{i}-{\EE{y(\cov^\top\e_i)\cov}}}=O\left(\sqrt{\frac{\tr{\Sigma_i}\log(n)}{m_0}}+\sqrt{\opnorm{\Sigma_i}\varepsilon}+\sqrt{\frac{\log(n/\delta)}{m_0}}\right),\label{eq:hlz_intro}
    \end{align}
    where we have used $n\opnorm{A}\geq\tr{A}$ for an $n\times n$ matrix. We compute upper-bounds on $\tr{\Sigma_i}$ and $\opnorm{\Sigma_i}$. Developing the expression for the covariance and using lemma \ref{lemma:moments_symmetric} (3.), we arrive at
    \begin{align*}
        \Sigma_i&=\left(12(\boldsymbol{x}_i^{\boldsymbol{*}})^2\norm{\truex}^2+3\norm{\truex}^4\right)I_n+5\norm{\truex}^4\e_i\e_i^\top+22\boldsymbol{x}_i^{\boldsymbol{*}}\norm{\truex}^4\left(\e_i(\truex)^\top+\truex\e_i^\top\right)\\&+\left(20(\boldsymbol{x}_i^{\boldsymbol{*}})^2+12\norm{\truex}^2\right)\truex(\truex)^\top+\sigma^2\left(I_n+\e_i\e_i^\top\right).
    \end{align*}
    Using $\tr{\boldsymbol{u}\boldsymbol{v}^\top}=\boldsymbol{u}^\top\boldsymbol{v}$, $\opnorm{\boldsymbol{u}\boldsymbol{v}^\top}=\norm{\boldsymbol{u}}\norm{\boldsymbol{v}}$, the triangle and Cauchy-Schwarz inequalities, we have:
    \begin{align*}
        \tr{\Sigma_i}&=O\left(n(\norm{\truex}^4+\sigma^2)\right),\\
        \opnorm{\Sigma_i}&=O\left(\norm{\truex}^4+\sigma^2\right).
    \end{align*}
    Substituting these into (\ref{eq:hlz_intro}), for each $i\in[n]$, with probability at least $1-\delta/(2n)$:
    \begin{align*}
        \norm{\boldsymbol{\mu}_{i}-{\EE{y(\cov^\top\e_i)\cov}}}=O\left(\!\norm{\truex}^2\sqrt{1+\sigma^2/\norm{\truex}^4}\left(\sqrt{\frac{n\log(n)}{m_0}}\!+\!\sqrt{\varepsilon}\!+\!\sqrt{\frac{\log(n/\delta)}{m_0}}\right)\!\right).
    \end{align*}
    By a union bound, with probability at least $1-\delta$:
    \begin{align*}
        \sqrt{\!\sum_{i=1}^n\!\norm{\boldsymbol{\mu}_{i}\!-\!{\EE{y\cov\cov^\top}}\e_i}^2}=\! O\!\left(\!\norm{\truex}^2\sqrt{1\!+\!\frac{\sigma^2}{\norm{\truex}^4}}\!\left(\!\sqrt{\frac{n^2\!\log(n)}{m_0}}\!+\!\sqrt{\varepsilon n}\!+\!\sqrt{\frac{n\!\log(n/\delta)}{m_0}}\right)\!\right).
    \end{align*}
    As $\boldsymbol{\mu}_{i}$ and $\mathbb{E}[y\cov\cov^\top]\e_i$ are the $i$'th column of $\boldsymbol{Y}$ and  $\mathbb{E}[y\cov\cov^\top]$, respectively, using the definition of the Frobenius norm, the previous inequality becomes
    \begin{align*}
        \norm{\boldsymbol{Y}-{\EE{y\cov\cov^\top}}}=O\left(\norm{\truex}^2\sqrt{1+\frac{\sigma^2}{\norm{\truex}^4}}\left(\sqrt{\frac{n^2\log(n)}{m_0}}+\sqrt{\varepsilon n}+\sqrt{\frac{n\log(n/\delta)}{m_0}}\right)\right).
    \end{align*}
    By Lemma \ref{lemma:moments_symmetric} (2.) and the fact that the Frobenius norm dominates the operator norm, with probability at least $1-\delta$,
    \begin{align*}
        \opnorm{\boldsymbol{Y}\!\!-\!\!\left(\norm{\truex}^2I_n\!+\!2\truex(\truex)^\top\!\right)}\!\!\!=\! O\!\left(\!\!\norm{\truex}^2\!\!\sqrt{1\!+\!\frac{\sigma^2}{\norm{\truex}^4}}\!\left(\!\!\sqrt{\frac{n^2\!\log(n)}{m_0}}\!+\!\sqrt{\varepsilon n}\!+\!\sqrt{\frac{n\!\log(2n/\delta)}{m_0}}\right)\!\!\right).
    \end{align*}
    As the operator norms of a matrix and its transpose are equal, an application of the triangle inequality implies that, with probability at least $1-2\delta$, the same upper-bound holds for $\opnorm{(\boldsymbol{Y}+\boldsymbol{Y}^\top)/2-(\norm{\truex}^2I_n+2\truex(\truex)^\top)}$.

     Next, we proceed as in the proof of Lemma 5 in \citep{Ma_2019}. Let $\wtx_0$ be an eigenvector of $\boldsymbol{\widetilde{Y}}$ for the eigenvalue $\lambda:=\largeev{\boldsymbol{Y}}$ and with norm $\norm{\wtx_0}=1$. Note that $-\wtx_0$ has the same property. The eigenvalues of $\norm{\truex}^2I_n+2\truex(\truex)^\top$ are either $3\norm{\truex}^2$ or $\norm{\truex}^2$, corresponding to $\truex/\norm{\truex}$ and a normal vector orthogonal to $\truex$, respectively. As the trace of this matrix is $(n+2)\norm{\truex}^2$, it follows that exactly one is $3\norm{\truex}^2$ and the rest are $\norm{\truex}^2$. Applying Davis-Kahan's Theorem (see, for example, Theorem 4.5.5 in \citep{HDP}), we obtain:
     \begin{align*}
         \dist{\norm{\truex}\wtx}{\truex}&\leq 2\sqrt{2}\norm{\truex}\frac{\opnorm{\frac{1}{2}\left(\boldsymbol{Y}+\boldsymbol{Y}^\top\right)-\left(\norm{\truex}^2I_n+2\truex(\truex)^\top\right)}}{2\norm{\truex}^2}\\
         &=O\left(\norm{\truex}\sqrt{1+\frac{\sigma^2}{\norm{\truex}^4}}\left(\sqrt{\frac{n^2\log (n)}{m_0}}+\sqrt{\varepsilon n}+\sqrt{\frac{n\log(n/\delta)}{m_0}}\right)\norm{\truex}\right).
     \end{align*}
     As $m_0=\Omega((1+\sigma^2/\norm{\truex}^4)\max\{n^2\log(n),n\log(n/\delta)\})$ and $\varepsilon=O((\norm{\truex}^4/\sigma^2)n^{-1})$, the hidden constants can be chosen in such a way that
     \begin{align}
         \dist{\norm{\truex}\wtx}{\truex}&\leq\frac{1}{18}\norm{\truex}.\label{eq:direction_init}
     \end{align}

    \bigskip
    \noindent
    \textbf{Recovering the scale of $\truex$:} Firstly, similar computations as in Lemma \ref{lemma:moments_symmetric} show that $\EE{y}=\EE{(\cov^\top\truex)^2}=\norm{\truex}^2$ and $\var{y}=2\norm{\truex}^4+\sigma^2$. By the guarantees of Proposition 2.1.1, with probability at least $1-\delta$:
    \begin{align*}
        \left|\widetilde{y}^2-\norm{\truex}^2\right|&= O\left(\sqrt{\var{y}}\left(\sqrt{\varepsilon}+\sqrt{\frac{\log(1/\delta)}{m_0}}\right)\right)\\
        &=O\left(\norm{\truex}^2\sqrt{1+\frac{\sigma^2}{\norm{\truex}^4}}\left(\sqrt{\varepsilon}+\sqrt{\frac{\log(1/\delta)}{m_0}}\right)\right).
    \end{align*}
    The value of $\varepsilon$ and the choice of $m_0$ guarantee that, in the previous line, the factor of $\norm{\truex}^2$ is less than $1$. Then, an application of the reversed triangle inequality guarantees that $\widetilde{y}^2>0$, so that
    \begin{align*}
        \left|\norm{\x_0}^2-\norm{\truex}^2\right|=O\left(\norm{\truex}^2\sqrt{1+\frac{\sigma^2}{\norm{\truex}^4}}\left(\sqrt{\varepsilon}+\sqrt{\frac{\log(1/\delta)}{m_0}}\right)\right).
    \end{align*}
    Next, as for any $a,b,c>0$ with $b^2>c$, $|a^2-b^2|\leq c$ implies $|a-b|\leq b-\sqrt{b^2-c}$, we get to
    \begin{align*}
        \left|\norm{\x_0}-\norm{\truex}\right|=O\left(\norm{\truex}\left(1-\sqrt{1-\sqrt{1+\frac{\sigma^2}{\norm{\truex}^4}}\left(\sqrt{\varepsilon}+\sqrt{\frac{\log(1/\delta)}{m_0}}\right)}\right)\right).
    \end{align*}
    Once again, the value of $\varepsilon$ and the choice of $m_0$ guarantee that 
    \begin{align}
        \left|\norm{\x_0}-\norm{\truex}\right|\leq \frac{1}{18}\norm{\truex}.\label{eq:scale_init}
    \end{align}

    \bigskip
    \noindent
    \textbf{Combining the direction and the scale:} Using the triangle inequality, together with (\ref{eq:direction_init}) and (\ref{eq:scale_init}), we have that with probability at least $1-2\delta$:
    \begin{align*}
        \dist{\x_0}{\truex}&\leq \norm{\x_0-\norm{\truex}\wtx_0}+\dist{\norm{\truex}\wtx_0}{\truex}=\left|\norm{\x_0}-\norm{\truex}\right|+\dist{\norm{\truex}\wtx_0}{\truex}\\
        &\leq\frac{1}{9}\norm{\truex}.
    \end{align*}
\end{proof}

\subsection{Proof of Theorem 3.2.}
\begin{proof}
We aim to apply Proposition 2.2.2. First of all, note that indeed the random variable $y\cov$ has mean $0$ and, under the assumption that the noise $\varepsilon$ has bounded fourth moment, it holds that $\mathbb{E}[\norm{y\cov}^4]<\infty$. 

Next, we give an upper-bound on $\kappa_4$. Recall its definition:
\begin{align*}
    \kappa_4=\underset{\boldsymbol{v}\in\mathbb{R}^n,\boldsymbol{v}^\top \Sigma\boldsymbol{v}=1}{\sup}\mathbb{E}[(\boldsymbol{v}^\top(y\cov))^4]^{1/4},
\end{align*}
where $\Sigma=\covariance{y\cov}=(3\norm{\truex}^4+\sigma^2)I_n+12\norm{\truex}^2\truex(\truex)^\top$, using similar calculations as the ones in Lemma \ref{lemma:moments_symmetric}. We begin by upper-bounding the expectation in the supremum. Using the inequality $(a+b)^4\leq 8(a^4+b^4)$,
\begin{align*}
    \mathbb{E}[(\boldsymbol{v}^\top(y\cov))^4]&=\mathbb{E}[(y^4\boldsymbol{v}^\top\cov)^4]=\mathbb{E}[((\cov^\top\truex)^2+z)^4(\boldsymbol{v}^\top\cov)^4]\leq8\left(\mathbb{E}[(a^\top\truex)^8(\boldsymbol{v}^\top\cov)^4]+K_4^4\mathbb{E}[(\boldsymbol{v}^\top\cov)^4]\right)\\
    &= 35360 \norm{\truex}^4(\boldsymbol{v}^\top\truex)^4+47800\norm{\truex}^8\norm{v}^4+24 K_4^4\norm{v}^4\leq (83160\norm{\truex}^8+24K_4^4)\norm{\boldsymbol{v}}^4,
\end{align*}
where the quantities in the last equality can be obtained by performin similar computations to Lemma \ref{lemma:moments_symmetric}, while the last inequality is due to Cauchy-Schwarz. Thus,
\begin{align*}
    \kappa_4\leq (83160\norm{\truex}^8+24K_4^4)^{1/4}\underset{\boldsymbol{v}\in\mathbb{R}^n,\boldsymbol{v}^\top \Sigma\boldsymbol{v}=1}{\sup}\norm{v}.
\end{align*}

As $1=\boldsymbol{v}^\top\Sigma\boldsymbol{v}=(3\norm{\truex}^4+\sigma^2)\norm{\boldsymbol{v}}^2+12\norm{\truex}^2(\boldsymbol{v}^\top\truex)^2$, we have, using also $(a+b)^{1/4}\leq a^{1/4}+b^{1/4}$ for $a,b>0$,
\begin{align*}
    \kappa_4\leq \frac{16\norm{\truex}^2+3 K_4}{\sqrt{3\norm{\truex}^4+\sigma^2}}\leq 16+3\frac{K_4}{\norm{\truex}^2}.
\end{align*}
In particular, this imposes $\varepsilon=O(\norm{\truex}^8/K_4^4)$, which is verified by the conditions of this Theorem. Using $r(\Sigma)\leq n$, the same is true for the condition on $m$. Thus, applying Proposition 2.2.2 and $\sigma^2\leq K_4^2$, we have with probability $1-\delta$,
\begin{align*}
    \opnorm{\boldsymbol{Y}-\covariance{y\cov}}=O\left(\norm{\truex}^4\left(1+\frac{K_4^4}{\norm{\truex}^8}\right)\left(\sqrt{\frac{n}{m}}+\sqrt{\varepsilon}+\sqrt{\frac{\log(1/\delta}{m}}\right)\right).
\end{align*}
From this point, the proof is identical to the Proof of Theorem 3.1. In particular, although the eigenvalues of $\covariance{y\cov}$ depend on the unknown $\sigma^2$, the difference between them does not and hence Davis-Kahan's Theorem can be applied without issues.
\end{proof}

\subsection{Proof of Theorem 3.3}
\begin{proof}
    We prove by induction that all iterates $(\x_t)_{t=0}^{T-1}$ remain in the ball centred at $\truex$ with radius $\norm{\truex}/9$, which will allow us to derive guarantees for the last iterate.

    \bigskip
    \noindent
    \textbf{The induction:} We assume without loss of generaty that $\dist{\x_0}{\truex}=\norm{\x_0-\truex}$ (otherwise, the proof follows identically by replacing $\truex$ with $-\truex$). The base case follows from the conditions of the theorem. Also, by the reversed triangle inequality,
    \begin{align*}
        \norm{\truex}-\norm{\x_0}\leq\frac{1}{9}\norm{\truex}\quad\implies\quad \norm{\truex}^2\leq\frac{81}{64}\norm{\x_0}^2.
    \end{align*}
    This guarantees that the step size satisfies $\eta=128/{(981\norm{\x_0}^2)}\leq{18}/({109\norm{\truex}^2})=2/{(\alpha+\beta)}$.
    
    Next, assume that for some $t\in\{0,1,\ldots,T-1\}$, $\norm{\x_t-\truex}\leq\norm{\truex}/9$. According to Lemma 2.2, the next iterate will satisfy the inequality (7), in which we will now determine $A(\wm,\delta,\varepsilon)$ and $B(\wm,\delta,\varepsilon)$. Towards this objective, let $\Sigma=\var{((\cov^\top\x_t)^2-y)(\cov^\top\x_t)\cov}$. According to Proposition 2.2.1, with probability at least $1-\delta$,
    \begin{align}
        \norm{\g_t-\nabla r(\x_t)}=O\left(\sqrt{\frac{\tr{\Sigma}\log(n)}{\wm}}+\sqrt{\opnorm{\Sigma}\varepsilon}+\sqrt{\frac{\opnorm{\Sigma}\log(1/\delta)}{\wm}}\right),\label{ineq:cond_on_x}
    \end{align}
    where we once again use $r(\Sigma)\leq n$. This holds conditionally on $\x_t$, as a fresh sample was used for computing $\g_t$. By an application of the tower law, (\ref{ineq:cond_on_x}) holds in general with probability at least $1-\delta$. Using the bounds in Lemma \ref{lemma:moments_symmetric} (5.) and the inequality $\sqrt{a+b}\leq\sqrt{a}+\sqrt{b}$, we have that, with probability at least $1-\delta$,
    \begin{align*}
        \norm{\g_t-\nabla r(\x_t)}\leq {A(\wm,\delta,\varepsilon)}\norm{\x_t-\truex}+{B(\wm,\delta,\varepsilon)},
    \end{align*}
    where $A(\wm,\delta,\varepsilon)$ and $B(\wm,\delta,\varepsilon)$ are defined as
    \begin{align*}
        A(\wm,\delta,\varepsilon)&=O\left(\norm{\truex}^2\left(\sqrt{\frac{n\log(n)}{\wm}}+\sqrt{\varepsilon}+\sqrt{\frac{\log(1/\delta)}{\wm}}\right)\right),\\ 
        B(\wm,\delta,\varepsilon)&=O\left(\sigma\norm{\truex}\left(\sqrt{\frac{n\log(n)}{\wm}}+\sqrt{\varepsilon}+\sqrt{\frac{\log(1/\delta)}{\wm}}\right)\right).
    \end{align*}
    Thus, according to Lemma 2.2, with probability at least $1-\delta$,
    \begin{align}
        \norm{\x_{t+1}-\truex}\leq \left(\sqrt{1-\frac{2\eta\alpha\beta}{\alpha+\beta}}+\eta A(\wm,\delta,\varepsilon)\right)\norm{\x_t-\truex}+\eta B(\wm,\delta,\varepsilon).\label{ineq:grad_contraction}
    \end{align}
    Next, we show $\sqrt{1\!-\!2\eta\alpha\beta/(\alpha+\beta)}\leq \!87/100$, $\eta A(\wm,\delta,\varepsilon)\leq 3/100$, and $\eta B(\wm,\delta,\varepsilon)\leq\norm{\truex}\!/90$ for the chosen values of $\wm$ and $\varepsilon$. Together with $\norm{\x_t-\truex}\leq\norm{\truex}/9$ and (\ref{ineq:grad_contraction}), these will imply that $\norm{\x_{t+1}-\truex}\leq\norm{\truex}/9$ (with probability at least $1-\delta$).
    \begin{itemize}
        \item $\sqrt{1-2\eta\alpha\beta/(\alpha+\beta)}\leq 87/100$: this is a consequence of $\norm{\x_0}^2\leq 100\norm{\truex}^2/81$, which is obtained from $\norm{\x_t-\truex}\leq\norm{\truex}/9$ by an application of the reversed triangle inequality.
        \item $\eta A(\wm,\delta,\varepsilon)\leq 3/100$: This follows by noting that
        \begin{align}
            \eta A(\wm,\delta,\varepsilon)&\leq\frac{2}{\alpha+\beta}A(\wm,\delta,\varepsilon)=O\left(\sqrt{\frac{n\log(n)}{\wm}}+\sqrt{\varepsilon}+\sqrt{\frac{\log(1/\delta)}{\wm}}\right).\label{ineq:etaA}
        \end{align}
        Clearly, for $\wm=\Omega(\max\{n\log(n),\log(1/\delta)\})$ and $\varepsilon$ a small enough constant, which is satisfied by the initial constraint on $\varepsilon$, the right-hand side of (\ref{ineq:etaA}) is at most $3/100$.
        \item $\eta B(\wm,\delta,\varepsilon)\leq\norm{\truex}/90$: As in the previous bullet point, this holds for $\varepsilon=O(\norm{\truex}^4/\sigma^2)$ and $\wm=\Omega(\max\{n\log(n),\log(1/\delta)\}\sigma^2/\norm{\truex}^4)$.
    \end{itemize}
    This concludes the inductive part of the proof.

    \bigskip
    \noindent
    \textbf{Guarantees for $\x_t$:} We have proved that on an event with probability at least $1-T\delta$, equation (\ref{ineq:grad_contraction}) holds for every $t\in\{0,1,\ldots,T-1\}$. Using the inequality $\sqrt{1-a}\leq1-a/2$, valid for any $a\in[0,1]$, it follows that for every $t\in\{0,1,\ldots,t-1\}$:
    \begin{align*}
        \norm{\x_{t+1}-\truex}\leq\left(1-\eta\left(\frac{\alpha\beta}{\alpha+\beta}- A(\wm,\delta,\varepsilon)\right)\right)\norm{\x_t-\truex}+\eta B(\wm,\delta,\varepsilon).
    \end{align*}
    Iterating this over $t\in\{0,1,\ldots,T-1\}$ and using $\norm{\x_0-\truex}\leq\norm{\truex}/9$, we have
    \begin{align}
        \norm{\x_T-\truex}&\leq\left(\!1\!-\!\eta \left(\frac{\alpha\beta}{\alpha+\beta}-A(\wm,\delta,\varepsilon)\right)\!\right)^T\!\norm{\x_0-\truex}+\frac{\eta B(\wm,\delta,\varepsilon)}{1\!-\!\left(1-\eta\left(\frac{\alpha\beta}{\alpha+\beta}-A(\wm,\delta,\varepsilon)\right)\right)}\nonumber\\
        &\leq\frac{1}{9}\exp{\left(-\eta T\left(\frac{\alpha\beta}{\alpha+\beta}-A(\wm,\delta,\varepsilon) \right)\right)}\norm{\truex}+\frac{ B(\wm,\delta,\varepsilon)}{\frac{\alpha\beta}{\alpha+\beta}-A(\wm,\delta,\varepsilon)}\label{ineq:gradient_partial_result}
    \end{align}
    We bound each term of (\ref{ineq:gradient_partial_result}) separately. As $\alpha\beta/(\alpha+\beta)\!=\!\Theta(\norm{\truex}^2)$ and $A(\wm,\delta,\varepsilon)\!=\!O(\norm{\truex}^2\sqrt{\varepsilon})$ for $\wm=O(\max\{n\log(n),\log(1/\delta)\})$, it holds that $\alpha\beta/(\alpha+\beta)-A(\wm,\delta,\varepsilon)=\Omega(\norm{\truex}^2(1-\varepsilon))$ and this gives that for some constants $C_3'$ and $C_4$
    \begin{align}
        \frac{1}{9}\exp{\left(-\eta T\left(\frac{\alpha\beta}{\alpha+\beta}-A(\wm,\delta,\varepsilon) \right)\right)}\norm{\truex}\leq C_3'\exp\left(-C_4\eta T\norm{\truex}^2(1-\sqrt{\varepsilon})\}\right)\norm{\truex}.\label{ineq:a_estimation}
    \end{align}
    Further, when $\varepsilon$ is a small enough constant, which is covered by the assumptions of our theorems,
    $A(\wm,\delta,\varepsilon)=O(\norm{\truex}^2)$, so that $\alpha\beta/(\alpha+\beta)-A(\wm,\delta,\varepsilon)=\Omega(\norm{\truex}^2)$. This leads to
    \begin{align}
        \frac{ B(\wm,\delta,\varepsilon)}{\frac{\alpha\beta}{\alpha+\beta}-A(\wm,\delta,\varepsilon)}\leq C_3''\frac{\sigma}{\norm{\truex}^2}\left(\sqrt{\frac{n\log(n)}{\wm}}+\sqrt{\frac{\log(1/\delta)}{\wm}}+\sqrt{\varepsilon}\right)\norm{\truex} \label{ineq:b_estimation}
    \end{align}
    for some appropriate constant $C_3''$. Plugging (\ref{ineq:a_estimation}) and (\ref{ineq:b_estimation}) into (\ref{ineq:gradient_partial_result}), we arrive to the conclusion
    \begin{align*}
        \frac{\norm{\x_T-\truex}}{\norm{\truex}}\!\leq\! C_3\!\left(\!\exp\left(\!-C_4\eta T\norm{\truex}^2\!(1\!-\!\sqrt{\varepsilon})\!\right)\!+\!\frac{\sigma}{\norm{\truex}^2}\!\left(\!\sqrt{\frac{n\log(n)}{\wm}}\!+\!\sqrt{\frac{\log(1/\delta)}{\wm}}\!+\!\sqrt{\varepsilon}\!\right)\!\right).
    \end{align*}
\end{proof}

\section{Technical Lemmas}

\subsection{The gradient and Hessian of the population risk}
\begin{lemma}
    \label{lemma:grad_and_hessian}
    The expressions for the gradient and Hessian of the population risk (3) are given by
    \begin{align}
        \nabla r(\x)&=3\norm{\x}^2\x-\norm{\truex}^2\x-2((\truex)^\top\x)\truex, \label{eq:gradient}\\ 
        \nabla^2 r(\x)&= 3\left(2\x\x^\top+\norm{\x}^2I_n\right)-\left(\norm{\truex}^2 I_n+2\truex(\truex)^\top\right).\label{eq:hessian}
    \end{align}
\end{lemma}
\begin{proof}
    Fix $\x$ and $\truex$. Exchanging differentiation and expectation, we have
    \begin{align*}
        \nabla r(\x)&=\EE{\left((\cov^\top \x)^2-y\right)(\cov^\top \x)\cov}=\EE{\left(\cov^\top \x)^2-(\cov^\top\truex)^2-z\right)(\cov^\top \x)\cov}\\
        &=\EE{\left(\cov^\top \x)^2-(\cov^\top\truex)^2\right)(\cov^\top \x)\cov},
    \end{align*}
    where for the last line we have used the fact that $\EE{z}=0$ and $z$ is independent of $\cov$. 
    
    Using Lemma \ref{lemma:moments_symmetric} (1.), $\EE{(\cov^\top\truex)^2(\cov^\top\x)\cov}=2((\truex)^\top\x)\truex+\norm{\truex}^2\x$. Letting $\truex:=\x$, we also get $\EE{(\cov^\top\x)^3\cov}=3\norm{\x}^2\x$. Putting these two expectation together, we arrive at (\ref{eq:gradient}). Differentiating again, we obtain the expression of the Hessian (\ref{eq:hessian}).
\end{proof}

\begin{lemma}
\label{lemma:grad_and_hessian_non-zero}
    The expressions for the gradient and Hessian of the population risk (8) are given by
    \begin{align*}
        \nabla r_{\mathrm{new}}(\x)&=2\norm{\x}^2\x-2(\x^\top\truex)\x,\\
        \nabla^2 r_{\mathrm{new}}(\x)&=2\norm{\truex}^2I_n+4\x\x^\top-2\truex(\truex)^\top.
    \end{align*}
\end{lemma}

\subsection{Quantities appearing in the proofs of Theorems 3.1 and 3.3}

\begin{lemma}
    \label{lemma:moments_symmetric}
    Let $\x$ and $\truex$ be fixed vectors in $\mathbb{R}^n$, $\cov\sim\mathcal{N}(\zero,I_n)$, $z$ a random variable with mean $0$ and variance $\sigma^2$, and $y=(\cov^\top\truex)^2+z$. Then:
    \begin{enumerate}[leftmargin=*]
        \item $\EE{(\cov^\top\truex)^2(\cov^\top\x)\cov}=2((\truex)^\top\x)\truex+\norm{\truex}^2\x$;
        \item $\EE{y\cov\cov^\top}=\norm{\truex}^2 I_n+2\truex(\truex)^\top$;
        \item $\EE{(\cov^\top\x)^2(\cov^\top\truex)^4\cov\cov^\top}=v_1I_n+v_2\x\x^\top+v_3\truex\x^\top+v_4\x(\truex)^\top+v_5\truex(\truex)^\top$, where
        \begin{align*}
            v_1&=12(\x^\top\truex)^2\norm{\truex}^2+3\norm{\truex}^4\norm{\x}^2,\\
            v_2&=6\norm{\truex}^4,\\
            v_3&=24(\x^\top\truex)\norm{\truex}^2,\\
            v_4&=24(\x^\top\truex)\norm{\truex}^2,\\
            v_5&=24(\x^\top\truex)^2+12\norm{\truex}^2\norm{\x}^2;
        \end{align*}
        \item if $\Sigma=\var{((\cov^\top\x)^2-y)(\cov^\top\x)\cov)}$ is the variance of the loss gradient and $\norm{\x-\truex}\leq 1/9$,
        \begin{align*}
            \tr{\Sigma}&\leq 525n\norm{\x-\truex}^2\norm{\truex}^4+6n\sigma^2\norm{\truex}^2,\\
            \opnorm{\Sigma}&\leq 525\norm{\x-\truex}^2\norm{\truex}^4+6\sigma^2\norm{\truex}^2.
        \end{align*}
    \end{enumerate}
\end{lemma}

\begin{proof}
    \begin{enumerate}[leftmargin=*]
        \item Let $A\in\mathbb{R}^{n\times n}$ be an orthogonal matrix ($AA^\top=A^\top A=I_n$) such that $A\truex=\norm{\truex}\e_1$ and $A\x=\norm{\x}(s_1\e_1+s_2\e_2)$, where $s_1,s_2\in\mathbb{R}$ satisfy $s_1^2+s_2^2=1$. Also, note that $\cov$ and $A\cov$ are both distributed as $\mathcal{N}(\zero,I_n)$. Then:
    \begin{align*}
       \EE{(\cov^\top \truex)^2(\cov^\top \x)\cov}&= \EE{(\cov^\top A^\top A\truex)^2(\cov^\top A^\top A\x)A^\top A\cov}\\
       &=\norm{\truex}^2\norm{\x}A^\top\EE{(\cov^\top \e_1)^2(\cov^\top(s_1\e_1+s_2\e_2))\cov}\\
       &=\norm{\truex}^2\norm{\x}A^\top\EE{a_1^2(s_1a_1+s_2a_2)\cov}\\
       &=\norm{\truex}^2\norm{\x}A^\top\mathbb{E}\begin{pmatrix}
           s_1a_1^4+s_2a_1^3a_2\\
           s_1a_1^3a_2+s_2a_1^2a_2^2\\
           s_1a_1^3a_3+s_2a_1^2a_2a_3\\
           \cdots\\
           s_1a_1^3a_n+s_2a_1^2a_2a_n
       \end{pmatrix}\\
       &=\norm{\truex}^2\norm{\x}A^\top(3s_1\e_1+s_2\e_2)\\
       &=\norm{\truex}^2\norm{\x}A^\top(2s_1\e_1+s_1\e_1+s_2\e_2)\\
       &=2((\truex)^\top\x)\truex+\norm{\truex}^2\x,
    \end{align*}
    where for the last line we used $\truex=\norm{\truex}A^\top\e_1$, $\x=\norm{\x}A^\top(s_1\e_1+s_2\e_2)$.
    \item We apply the same technique as in the previous part. In particular, $A$ is an orthonormal matrix with $A\truex=\norm{\truex}\e_1$.
    \begin{align*}
        \EE{y\cov\cov^\top}&=\EE{\left((\cov^\top\truex)^2+z\right)\cov\cov^\top}=\EE{(\cov^\top\truex)^2\cov\cov^\top}=\norm{\truex}^2A^\top\EE{(\cov^\top\e_1)^2\cov\cov^\top}A\\
        &=\norm{\truex}^2A^\top(I_n+2\e_1\e_1^\top)A=\norm{\truex}^2I_n+2\truex(\truex)^\top.
    \end{align*}
    \item Again, consider an orthonormal matrix $A\in\mathbb{R}^{n\times n}$ such that $A\truex=\norm{\truex}\e_1$ and $A\x=\norm{\x}(s_1\e_1+s_2\e_2)$, with $s_1,s_2\in\mathbb{R}$, $s_1^2+s_2^2=1$. Then, note that
    \begin{align*}
        &\EE{(\cov^\top\truex)^2(\cov^\top\truex)^4\cov\cov^\top}\\&=\norm{\truex}^4\norm{\x}^2\EE{(\cov^\top(s_1\e_1+s_2\e_2))^2(\cov^\top\e_1)^4\cov\cov^\top}\\
        &=\norm{\truex}^4\norm{\x}^2\EE{(s_1a_1+s_2a_2)^2a_1^4\cov\cov^\top}\\
        &=\norm{\truex}^4\norm{\x}^2\begin{pmatrix}
            105s_1^2+15s_2^2 & 30s_1s_2 & 0 & \cdots & 0\\
            30s_1s_2 & 15s_1^2+9s_2^2 & 0 & \cdots & 0\\
            0 & 0 & 15s_1^2+9s_2^2 & \cdots & 0\\
            \vdots & \vdots & \vdots & \ddots & \vdots\\
            0 & 0 & 0 & \cdots & 15s_1^2+3s_2^2
        \end{pmatrix}\\
        &=\norm{\truex}^4\norm{\x}^2( (15s_1^2+3s_2^2)I_n+6(s_1\e_1+s_2\e_2)(s_1\e_1+s_2\e_2)^\top\\
        &+24s_1\e_1(s_1\e_1+s_2\e_2)^\top+24s_1(s_1\e_1+s_2\e_2)\e_1^\top+(36s_1^2+12s_2^2)\e_1\e_1^\top ).
    \end{align*}
    The conclusion will follow by using $s_1^2+s_2^2=1$, $s_1=\e_1^\top(s_1\e_1+s_2\e_2)$ and the inverse transformations $\truex=\norm{\truex}A^\top\e_1$, $\x=\norm{\x}A^\top(s_1\e_1+s_2\e_2)$.
    \item We start by computing $\Sigma$. Recall that the fact that the noise is symmetric with variance $\sigma^2$ and independent of the covariates:
    \begin{align*}
        \Sigma&=\var{\left((\cov^\top\x)^2-y\right)(\cov^\top\x)\cov}\\
        &=\EE{\left((\cov^\top\x_t)^2-(\cov^\top\truex)^2-z\right)^2(\cov^\top\x)^2\cov\cov^\top}-\nabla r(\x)\nabla r(\x)^\top\\
        &=\EE{(a^\top\x)^6\cov\cov^\top}+\EE{(\cov^\top\truex)^4(\cov^\top\x)^2\cov\cov^\top}+\sigma^2\EE{(\cov^\top\x)^2\cov\cov^\top}\\
        &-2\EE{(\cov^\top\x)^4(\cov^\top\truex)^2)\cov\cov^\top}-\nabla r(\x)\nabla r(\x)^\top.
    \end{align*}
    We can use the expression derived in part (4.) to arrive at expressions for the first, second and fourth term on the previous line. The last term has already been calculated in Lemma \ref{lemma:grad_and_hessian}, and the same computations as in the proof of (2.) give $\EE{(\cov^\top\truex)^2\cov\cov^\top}=\norm{\x}^2I_n+2\x\x^\top$. Putting all of these together,
    \begin{align*}
        \Sigma&=(15\norm{\x}^6+12(\x^\top\truex)^2\norm{\truex}^2+3\norm{\truex}^4\norm{\x}^2-24(\x^\top\truex)^2\norm{\x}^2-6\norm{\x}^4\norm{\truex}^2)I_n\\
        &+(91\norm{\x}^4+5\norm{\truex}^4-48(\x^\top\truex)^2-18\norm{\x}^2\norm{\truex}^2)\x\x^\top\\
        &+(20(\x^\top\truex)^2+12\norm{\truex}^2\norm{\x}^2-12\norm{\x}^4)\truex(\truex)^\top\\
        &+(22(\x^\top\truex)\norm{\truex}^2-42(\x^\top\truex)\norm{\x}^2)(\truex\x^\top+\x(\truex)^\top)\\
        &+\sigma^2(\norm{\x}^2I_n+2\x\x^\top).
    \end{align*}
    Next, the expression of $\Sigma$ in terms of $\truex$ and $\h:=\x-\truex$: 
    \begin{align*}
        \Sigma&=w_1 I_1+w_2\truex(\truex)^\top+w_3\h\h^\top+w_4(\truex\h^\top+\h(\truex)^\top)+\sigma^2 A,
    \end{align*}
    where the expressions of $w_1,w_2,w_3,w_4\in\mathbb{R}$ and $A\in\mathbb{R}^{n\times n}$ are
    \begin{align*}
        w_1&=15\norm{\h}^6+72(\h^\top\truex)^3+90\norm{\h}^4(\h^\top\truex)+156\norm{\h}^2(\h^\top\truex)^2+39\norm{\h}^4\norm{\truex}^2\\
        &+12\norm{\h}^2\norm{\truex}^4+48\norm{\truex}^2(\h^\top\truex)^2+108\norm{\h}^2\norm{\truex}^2(\h^\top\truex),\\
        w_2&=69\norm{\h}^4+80(\h^\top\truex)^2+192(\h^\top\truex)\norm{\h}^2+48\norm{\truex}^2\norm{\h}^2,\\
        w_3&=81\norm{\h}^4+276(\h^\top\truex)^2+20\norm{\truex}^4+324(\h^\top\truex)\norm{\h}^2+192(\h^\top\truex)\norm{\truex}^2\\ &+144\norm{\h}^2\norm{\truex}^2,\\
        w_4&=81\norm{\h}^4+192(\h^\top\truex)^2+88(\h^\top\truex)\norm{\truex}^2+282(\h^\top\truex)\norm{\h}^2+102\norm{\h}^2\norm{\truex}^2,\\
        A&=(\norm{\truex}^2+2(\h^\top\truex)+\norm{\h}^2)I_n+2\truex(\truex)^\top+2\h\h^\top+2(\truex\h^\top+\h(\truex)^\top).
    \end{align*}
    To compute an upper-bound on $\tr{\Sigma}$ we use the property $\tr{\boldsymbol{u}\boldsymbol{v}^\top}=\boldsymbol{u}^\top\boldsymbol{v}$, the Cauchy-Schwarz inequality and the assumption $\norm{\h}\leq\norm{\truex}/9$:
    
    \begin{align*}
        \tr{\Sigma}&=(81+15n)\norm{\h}^6+(384+72n)(\h^\top\truex)^3+(486+90n)\norm{\h}^4(\h^\top\truex)\\
        &+(840+156n)\norm{\h}^2(\h^\top\truex)^2+(213+39n)\norm{\h}^4\norm{\truex}^2+(68+12n)\norm{\h}^2\norm{\truex}^4\\
        &+(256+48n)(\h^\top\truex)^2\norm{\truex}^2+(588+108n)(\h^\top\truex)\norm{\h}^2\norm{\truex}^2\\
        &+(n+2)\sigma^2(\norm{\h}^2+2\h^\top\truex+\norm{\truex}^2)\\
        &\leq (81+15n)\norm{\h}^6+(486+90n)\norm{\h}^5\norm{\truex}+(1053+195n)\norm{\h}^4\norm{\truex}^2\\
        &+(972+180n)\norm{\h}^3\norm{\truex}^3+(324+60n)\norm{\h}^2\norm{\truex}^4\\&
        +(n+2)\sigma^2(\norm{\h}^2+2\norm{\h}\norm{\truex}+\norm{\truex}^2)\\
        &\leq\left(\frac{81\!+\!15n}{9^4}\!+\!\frac{486\!+\!90n}{9^3}\!+\!\frac{1053\!+\!195n}{9^2}\!+\!\frac{972\!+\!180n}{9}\!+\!324\!+\!60n\right)\norm{\h}^2\norm{\truex}^4\\
        &+(n+2)\sigma^2\left(\frac{1}{81}+\frac{2}{9}+1\right)\norm{\truex}^2\\
        &\leq (442+83n)\norm{\h}^2\norm{\truex}^4+2(n+2)\sigma^2\norm{\truex}^2\\
        &\leq 525n\norm{\h}^2\norm{\truex}^4+6n\sigma^2\norm{\truex}^2.
    \end{align*}
    Taking the operator norm ($\opnorm{\boldsymbol{u}\boldsymbol{v}^\top}=\opnorm{\boldsymbol{u}}\opnorm{\boldsymbol{v}}$), using once again the Cauchy-Schwarz inequality and the assumption $\norm{\h}\leq\norm{\truex}/9$, we similarly arrive at
    \begin{align*}
       \opnorm{\Sigma}\leq 525\norm{\h}^2\norm{\truex}^4+6\sigma^2\norm{\truex}^2.
    \end{align*}
    \end{enumerate}
\end{proof}

\begin{lemma}
    \label{lemma:moments_non-symmetric}
    Let $\x$ and $\truex$ be fixed vectors in $\mathbb{R}^n$, $\covb,\covc\sim\mathcal{N}(\zero,I_n)$ independently, $\zeta$ a random variable with mean $0$ and variance $\sigma^2/2$, and $\upsilon=(\truex)^\top\covb\covc^\top\truex+\zeta$. Then:
    \begin{enumerate}[leftmargin=*]
        \item $\EE{\upsilon \covb\covc^\top}=\truex(\truex)^\top$;
        \item $\textup{Cov}({\upsilon(\covc^\top \e_i)\covb})=(\norm{\truex}^4+2(\boldsymbol{x}_1^{\boldsymbol{*}})^2\norm{\truex}^2)I_n+(2\norm{\truex}^2+3(\boldsymbol{x}_1^{\boldsymbol{*}})^2)\truex(\truex)^\top+\sigma^2I_n/2$;
        \item if $\Sigma=\var{(\x^\top\covb\covc^\top\x-\upsilon)(\covb\covc^\top+\covc\covb^\top)\x}$ is the variance of the loss gradient, then
        \begin{align*}
            \Sigma&=(6\norm{\x}^6+2\norm{\truex}^4\norm{\x}^2+4(\x^\top \truex)^2\norm{\truex}^2-12(\x^\top\truex)^2\norm{\x}^2)I_n\\
            &+(26\norm{\x}^4+2\norm{\truex}^4-16(\x^\top\truex)^2)\x\x^\top\\
            &+(12(\x^\top\truex)^2+4\norm{\truex}^2\norm{\x}^2-4\norm{\x}^4)\truex(\truex)^\top\\
            &+(4(\x^\top\truex)\norm{\truex}^2-16(\x^\top\truex)\norm{\x}^2)(\truex\x^\top+\x(\truex)^\top)\\
            &+\sigma^2(\norm{\x}^2 I_n+\x\x^\top).
        \end{align*}
        In particular, if $\norm{\x-\truex}\leq 1/6$,
        \begin{align*}
            \tr{\Sigma}&\leq 242n\norm{\x-\truex}^2\norm{\truex}^4+3n\sigma^2\norm{\truex}^2,\\
            \opnorm{\Sigma}&\leq242\norm{\x-\truex}^2\norm{\truex}^4+3\sigma^2\norm{\truex}^2.
        \end{align*}
    \end{enumerate}
\end{lemma}

\begin{proof}
    The proof of these results is identical to the one of Lemma \ref{lemma:moments_symmetric}.
\end{proof}

\end{document}